\documentclass{article}
\usepackage[utf8]{inputenc}
\usepackage{geometry}
\usepackage{amsmath,bm,amsthm}
\usepackage[colorlinks=true,allcolors=blue]{hyperref}
\usepackage{bbm}

\usepackage{biblatex}
\usepackage{amssymb,amsfonts}
\usepackage[ruled,vlined]{algorithm2e}
\usepackage{algorithmic}
\usepackage{graphicx}
\usepackage{textcomp}
\usepackage{enumitem}
\usepackage{xcolor}
\usepackage{float}
\usepackage{appendix}
\usepackage{subcaption}
\usepackage{caption}
\usepackage{wrapfig}
\usepackage{multirow}
\usepackage{graphicx}
\usepackage{wrapfig}

\def\BibTeX{{\rm B\kern-.05em{\sc i\kern-.025em b}\kern-.08em
    T\kern-.1667em\lower.7ex\hbox{E}\kern-.125emX}}
\newcommand{\bR}{\mathcal{R}}
\newcommand{\bH}{\mathcal{H}}
\newcommand{\bX}{\mathcal{X}}

\newcommand\numberthis{\addtocounter{equation}{1}\tag{\theequation}}
\newtheorem{theorem}{Theorem}
\newtheorem{lemma}{Lemma}

\newtheorem{proposition}{Proposition}
\newtheorem{assumption}{Assumption}

\newtheorem{note}{Remark}

\title{The Directional Bias Helps Stochastic Gradient Descent to Generalize in Kernel Regression Models}
\author{Yiling Luo, Xiaoming Huo, Yajun Mei}
\date{April 2022}

\begin{document}

\maketitle

\begin{abstract}
We study the Stochastic Gradient Descent (SGD) algorithm in nonparametric statistics: kernel regression in particular. 
The directional bias property of SGD, which is known in the linear regression setting, is generalized to the kernel regression. 
More specifically, we prove that SGD with moderate and annealing step-size converges along the direction of the eigenvector that corresponds to the largest eigenvalue of the Gram matrix. 
In addition, the Gradient Descent (GD) with a moderate or small step-size converges along the direction that corresponds to the smallest eigenvalue. 
These facts are referred to as the directional bias properties; they may interpret how an SGD-computed estimator has a potentially smaller generalization error than a GD-computed estimator.   
The application of our theory is demonstrated by simulation studies and a case study that is based on the FashionMNIST dataset.     
\end{abstract}

\section{Introduction}
The Stochastic Gradient Descent (SGD) is a popular optimization algorithm that has a wide range of applications, including generalized linear model in statistics and deep Neural Network in machine learning. 
One main advantage of the SGD is the computational scalability due to low cost per iteration. 
Recent work also indicates that the SGD might also lead to outcomes that possess nice statistical properties under the linear regression framework, see \cite{ma2018power}. 


In this paper, we study the statistical properties of the SGD under nonparametric regression models. 
We focus on the Reproducing Kernel Hilbert Space (RKHS) model, which is popular in both statistics and machine learning communities and is often simply referred to as the ``kernel trick," see \cite{bartlett2021deep,wahba1990spline}. 
The kernel method can be applied in various domains such as image processing \cite{takeda2007kernel} and text mining \cite{greene2006practical}.

Our main approach is to analyze the directional bias of the SGD algorithm under the RKHS model, which can help us to explain why the outcome of SGD has good generalization properties. 
Directional bias, also referred to as implicit bias, means that an algorithm generates a solution path that is biased towards a certain direction, and it is also closely related to implicit regularization in deep learning \cite{NEURIPS2020_37740d59}. 
Directional bias also means that the algorithm prefers some directions over the others even though they may have the same objective function value. 
For example, paper \cite{wu2021direction} shows the directional bias of SGD and GD in the linear regression model, and analyzes the relationship between the directional bias and the generalization error in such a setting. 
There is no directional bias result in kernel regression, but one may expect it to exist and to explain the generalization performance of the outcome of SGD in the kernel regression model. 

The state-of-the-art result on the directional bias of SGD can be divided into two categories, based on their underlying techniques, mostly under the linear regression model. 
The first category is the stochastic gradient flow method where one assumes an infinitesimal step-size in SGD and thus the parameter dynamic follows a stochastic differential equation, see \cite{liu2018diffusion,ali2020implicit,pmlr-v125-blanc20a}. 
The second category is to analyze the discrete SGD sequence, which in general requires a moderate step-size such that the algorithm converges, see \cite{wu2021direction} for a example in linear regression. 
Our approach belongs to the second category. 

We want to point out that there are more research to study the directional bias of the Gradient Descent (GD) than for the SGD. 
For instance, for Neural Networks in the `lazy training' regime, paper \cite{cao2020understanding} shows that GD  converges in the direction of the smallest eigenvalue of the Neural Tangent Kernel; 
for high-dimension regression, papers  \cite{NEURIPS2019_5cf21ce3,zhao2019implicit} 
show that GD finds a sparse estimation.

\textbf{Our contributions} are two folded. 
First, we study the directional bias of (S)GD in a nonparametric regression model. 
Though the nonparametric regression is well studied in statistics, the directional bias is a relatively new concept \cite{wu2021direction}. 
Second, we unify the conditions to guarantee the directional bias of GD and SGD sequences. 
The main condition is the diagonally dominant Gram matrix, which covers a large class of kernel functions. 

Our result can shed new light on deep learning \cite{belkin2018understand}. 
By the state-of-the-art mathematical theory of Neural Networks (NN), kernel and/or nonparametric methods can approximate the functional space of neural networks, see for example the NTK theory \cite{Arthur2018}, and the Radon bounded variance space description for ReLU NN \cite{parhi2021kinds}. 
These phenomena can lead to interesting future research. 

\textbf{Organization}. 
The rest of this paper is organized as follows. 
In Section \ref{ch2:sec:2}, we give problem formulation, revisit the algorithms and state our assumption. 
In Section \ref{ch2:sec:3}, we state our main theorems, including the directional bias result and its implication for generalization.
In Section \ref{ch2:sec:4}, we perform numerical experiments.  
In Section \ref{ch2:sec:5}, we discuss our findings and propose some future research topics. 
The detailed proof and experiments are included in the appendix.

\section{Problem Formulation}\label{ch2:sec:2}
In Subsection \ref{ch2:sec:2.1} we define the kernel regression; 
in Subsection \ref{ch2:sec:2.2}, we present the SGD and GD algorithms; 
in Subsection \ref{ch2:sec:2.3}, we state our assumption for later analysis. 
We also provide a simple example to justify the assumption. 

\subsection{Kernel Regression}\label{ch2:sec:2.1}
Suppose that we have $n$ data pairs $\{\bm{x}_i,y_i\}_{i=1}^n$, where $y_i\in \bR$ is associated with $\bm{x}_i \in\bX \subset\bR^p$ through an unknown model $f(\bm{x}_i)$. The goal is to estimate the unknown model $f$ from the data. One solution is to minimize the empirical risk function
\begin{align}\label{ch2:eq:ERM}
    \min_{f} \frac{1}{n}\sum_{i = 1}^n \ell(y_i,f(\bm{x}_i)),
\end{align}
where $\ell$ is the loss function. A popular choice for regression task is the squared loss $\ell(y,\bm{x}) = \frac{1}{2}( y - f(\bm{x}))^2$. 

One can see that problem \eqref{ch2:eq:ERM} is not well-defined, as there are infinitely many solutions to $\forall i: f(\bm{x}_i) = y_i$, and some of them do not generalize for a new test data. One way to fix it is to restrict $f\in \bH$ and penalize $\|f\|_{\bH}$ for smoothness, where $\bH$ is a RKHS with reproducing kernel $K(\cdot,\cdot)$ and $\|\cdot\|_{\bH}$ is the Hilbert norm. Adding these restrictions and applying Representer Theorem, problem \eqref{ch2:eq:ERM} with the squared loss becomes
\begin{align}
\begin{split}\label{ch2:eq:reg_RKHS}
    \min_{\bm{\alpha}\in \mathcal{R}^n} &\frac{1}{2n}\sum_{i=1}^n(y_i - \bm{K}_i^T\bm{\alpha})^2 =\frac{1}{2n}\|\bm{y} - K\bm{\alpha}\|_2^2,
\end{split}
\end{align}
where $\bm{K}_i^T$ is the $i$th row of $K :=K(X,X) = (K(\bm{x}_i,\bm{x}_j))_{i,j}$. For a parameter $\bm{\alpha}$, the corresponding estimator in $\bH$ is $f(\cdot) = \sum_{i = 1}^n \alpha_iK(\bm{x}_i,\cdot):=\boldsymbol{\alpha}^T K(\cdot,X)$.

Now when $K$ is invertible, it is trivial that any algorithm on objective function \eqref{ch2:eq:reg_RKHS} converges at the unique minimizer $\bm{\hat{\alpha}} = K(X,X)^{-1}\bm{y}$, so the RKHS functional estimator is
\begin{align}\label{ch2:eq:esti}
    \hat{f}(\bm{x}) =K(\bm{x},X)^T K(X,X)^{-1}\bm{y},
\end{align}
where $K(\bm{x},X)^T = (K(\bm{x},\bm{x}_1),\ldots,K(\bm{x},\bm{x}_n))$. Estimator \eqref{ch2:eq:esti} is the minimum norm interpolant, i.e.:
\begin{align*}
    \arg\min_{f\in \bH} \{\|f\|_{\bH}:f(\bm{x}_i) = y_i,i = 1,\ldots,n\},
\end{align*}
whose properties are studied in \cite{liang2020JustInterpolate}.

In this work, we compare the convergence direction of SGD and GD to $\hat{\alpha}$. Specifically, we consider a two-stage SGD with a phase transition from a larger step-size to a decreased step-size. Note that this matches the training scheme people always use in practice for SGD algorithms: decreasing the step-size after training for a few epochs. For that purpose, in the following sections, we define the one-step SGD/GD update and state our assumptions and notations for analysis.

\subsection{One step SGD/GD update}\label{ch2:sec:2.2}
For objective function \eqref{ch2:eq:reg_RKHS}, denote the parameter estimation at $t$th step as $\alpha_t$, then SGD update $\bm{\alpha}_{t+1}$ as
\begin{align}
\begin{split}\label{ch2:eq:SGD_update}
\bm{\alpha}_{t+1} = \bm{\alpha}_{t} -\eta_t (\bm{K}_{i_t}^T \bm{\alpha}_t - y_{i_t})\cdot \bm{K}_{i_t},
\end{split}
\end{align}
where $i_t$ is uniformly random sampled from $\{1,\ldots,n\}$.

The GD update $\bm{\alpha}_{t+1}$ as
\begin{align}
\begin{split}\label{ch2:eq:GD_update}
\bm{\alpha}_{t+1} = \bm{\alpha}_{t} -\frac{\eta_t}{n} K^T(K \bm{\alpha}_t - \bm{y}).
\end{split}
\end{align}

\subsection{Assumptions and Notations}\label{ch2:sec:2.3}
We state our assumption on the Gram matrix as following:

\begin{assumption}[Diagonally dominant Gram matrix]\label{ch2:ass:1}
Denote by $K= K(X,X)$ the Gram matrix, we assume that $K$ is diagonally dominant. Specifically, suppose w.l.o.g. that $K_{1,1}\geq K_{2,2}\geq \ldots\geq K_{n,n}>0$, then we have for a small value $\tau$ that
\begin{equation*}
    |K_{i,j}|\leq \tau \ll K_{n,n}, \forall i\neq j.
\end{equation*}
\end{assumption}

\begin{note}\label{ch2:note:diag_dom_note1}
Diagonally dominant Gram matrix is common in kernel learning. Mathematically, one can justify that a Gram matrix is diagonally dominant by imposing proper assumptions on the kernel function $K(\cdot,\cdot)$ and the data distribution. 
One can see Appendix \ref{ch2:app:diagonal_dominance} for some examples.
\end{note}

\begin{note}\label{ch2:note:diag_dom_note2}
Thinking of the kernel function as the inner product of high-dimensional features, the resulting Gram matrix is diagonally dominant when the high-dimension features are sparse. 
This happens in a lot of practical problems \cite{scholkopf2002kernel,weston2003dealing}, for example, linear or string kernels being applied to text data \cite{greene2006practical}, 
domain-specific kernels being applied to image retrieval \cite{tao2004extended} and bioinformatics \cite{saigo2004protein}, and  the Global Alignment kernel being applied to most datasets \cite{cuturi2007kernel,cuturi2011fast}.
\end{note}

\begin{proposition}[Lemma 1 in \cite{wu2021direction}]
Consider the bilinear kernel $K(\bm{x},\bm{x}') := \langle\bm{x},\bm{x}'\rangle$. 
Assume the data $\bm{x}_i, i=1,\ldots,n$, are i.i.d. uniformly distributed on the unit sphere $S^{d-1}$, where $d\gg n$. 
When $d \geq 4 \log(2n^2/\delta)$ for some $\delta\in(0, 1)$. Then with probability at least $1 - \delta$, we have
$$  |K_{i,j}| = |\langle \bm{x}_i,\bm{x}_j\rangle| <\tilde{\tau} := \Tilde{\mathcal{O}}(1/\sqrt{d})\ll K_{n,n} = 1,\forall i \neq j.$$
\end{proposition}

Though commonly exists, the diagonal dominance is undesired in classification and clustering tasks. It indicates that the data points are dissimilar to each other, which means not enough information for classification/clustering. 
There are some efforts for solving the issue of diagonal dominance in these cases, see for example \cite{greene2006practical,kandola2003reducing}. But for the regression task, the diagonal dominance, in other words, the dissimilarity of data points, may have benefits. 
One can find similar conditions such as Restricted Isometry Property and $s$-goodness that describes linearly dissimilar features in a  regression literature \cite{candes2007dantzig,chen1994basis}. 
Such conditions are required for proving minimax optimality or exact recovery of a sparse signal in many settings. 
In our case, we adopt the dissimilarity concept and 
apply it to data points in high-dimensional nonlinear feature space. 
Later we will see that in our case, the directional bias drives SGD to select a good solution that generalizes well among all solutions of the same level of empirical loss. In this way, our SGD estimator benefits from the diagonal dominance.

\textbf{Notations}. 
We use the following notations throughout this paper. 
For the Gram matrix $K$, let $K_{i,j}$ denote its $(i,j)$th element. 
Denote $\lambda_i = K_{i,i} =K(\bm{x}_i,\bm{x}_i)$, and assume w.l.o.g. that $\lambda_1\geq\lambda_2\geq\ldots\geq\lambda_n$.  
Denote the $i$th column of $K$ as $\bm{K}_i$, let $K_{-1} = [K_2,\ldots,K_n]$.  
Denote $P_{-1}$ the projection onto column space of $K_{-1}$, and $P_1 = I - P_{-1}$. 
And denote $\gamma_1\geq\ldots\geq\gamma_n>0$ the eigenvalues of $K$ in non-increasing order.

\section{Main result}\label{ch2:sec:3}
The main results are presented in two subsections: 
Subsection \ref{ch2:sec:3.1} states the directional bias results of SGD and GD estimators, respectively; 
Subsection \ref{ch2:sec:3.2} shows that certain directional bias leads to good generalization performance, and applies this result to show that an outcome from SGD potentially generalizes better than an outcome from GD. 

\subsection{Directional bias }\label{ch2:sec:3.1}
By our assumption, $K$ will be full rank, (S)GD on \eqref{ch2:eq:reg_RKHS} converges at $\hat{\bm{\alpha}} = K^{-1}\bm{y}$. We are interested in the direction at which $\bm{\alpha}_t$ converges to $\hat{\bm{\alpha}}$, i.e., the quantity
$$
\bm{b}_t := \bm{\alpha}_t - \hat{\bm{\alpha}}.
$$
With Assumption \ref{ch2:ass:1} that the Gram matrix is diagonally dominant, we prove that a two-stage SGD has $\bm{b}_t$ converge in the direction that is aligned with the eigenvector associated with the largest eigenvalue of the Gram matrix $K$.

\begin{theorem}[Directional bias of an SGD-based estimator]\label{ch2:thm:informal1}
Assume Assumption \ref{ch2:ass:1} holds, run a two-stage SGD with a fixed step-size for each stage: stage 1 with step-size $\eta_1$ for steps $1,\ldots,k_1$, stage $2$ with step-size $\eta_2$ for steps $k_1 + 1,\ldots, k_2$, such that
\begin{align*}
    &\frac{2}{\lambda_1^2 - C_1 \sqrt{n}\tau} < \eta_1 <\frac{2}{\lambda_2^2 + C_2 \sqrt{n}\tau}, \\
    &\eta_2 < \frac{1}{\lambda_1^2 + C_3 \sqrt{n}\tau}, 
\end{align*}
where $C_1,C_2,C_3$ are constants. For a small $\epsilon > 0$ such that $n\tau < poly(\epsilon)$ there exists $k_1 = \mathcal{O}(\log\frac{1}{\epsilon})$ and $k_2$ such that
$$(1-2\epsilon)\gamma_1\leq E[\|K \bm{b}_{k_2}^{SGD}\|_2]/E[\|\bm{b}_{k_2}^{SGD}\|_2] \leq \gamma_1.$$ 
That is, $\bm{b}_{k_2}^{SGD}$ is close to the direction of eigenvector corresponding to the largest eigenvalue of $K$.
\end{theorem}
\begin{note}\label{ch2:note:05}
One should assume $\tau$ in Assumption \ref{ch2:ass:1} to be small enough for $\epsilon$ to be very small if one would like the resulting estimator $\bm{b}_{k_2}^{SGD}$  to have the direction that corresponds to the largest eigenvalue of $K$. Later we will see that if one only wants different directional bias of SGD and GD estimators, a moderate $\epsilon$ is allowed, the assumption on $\tau$ is not that strong.
\end{note}
The proof for Theorem \ref{ch2:thm:informal1} is in Appendix \ref{ch2:app:E}. Next, we see the different convergence direction of GD.

\begin{theorem}[Directional bias of a GD-based estimator]\label{ch2:thm:informal2}
Assume Assumption \ref{ch2:ass:1} holds, run GD with a fixed step-size $\eta$ such that 
\begin{align*}
   \eta<n/(\lambda_1 + n\tau)^2.
\end{align*}
For a $\epsilon' > 0$, let $k = \mathcal{O}(\log \frac{1}{\epsilon'})$, we have the GD estimator after $k$ steps satisfying:
\begin{align*}
    \gamma_n\leq \|K \bm{b}_k^{GD} \|_2/\|\bm{b}_k^{GD}\|_2\leq \sqrt{1 + \epsilon'}\gamma_n. 
\end{align*}
That is, $\bm{b}_k^{GD}$ is close to the direction that corresponds to the smallest eigenvalue of $K$.
\end{theorem}
\begin{note}\label{ch2:note:07}
The assumption (on $\tau$) is mild for differentiating the directional bias of SGD and GD. Comparing Theorem \ref{ch2:thm:informal1} and \ref{ch2:thm:informal2},when $\gamma_n < (1 - 2\epsilon)\gamma_1$, taking $k$ large enough we have
$$\frac{\|K \bm{b}_k^{GD} \|_2}{\|\bm{b}_k^{GD}\|_2} < \frac{E\|K \bm{b}_{k_2}^{SGD} \|_2}{ E\|\bm{b}_{k_2}^{SGD}\|_2}.$$
That is, one may expect $\bm{b}_{k_2}^{SGD}$ to be in the direction of larger eigenvalue compared with $\bm{b}_{k}^{GD}$. 
In the following subsection, we will see that the directional bias towards a larger eigenvalue of the kernel is good for generalization. 
That is, directional bias helps an SGD-based estimator to generalize. 
\end{note}
One can see the detailed proof of Theorem 
\ref{ch2:thm:informal2} in Appendix \ref{ch2:app:F}. 
Though assumption \ref{ch2:ass:1} appears in Theorem \ref{ch2:thm:informal2}, it is just used to bound the step-size so that GD converges; the diagonally dominant structure of $K$ is not required. 
Moreover, the choice of $\epsilon'$ is independent of $\tau$, then for an arbitrarily small $\epsilon' >0$, run GD long enough then the theorem will apply. 
The estimator $\bm{b}_{k}^{GD}$ can be arbitrarily close to the eigenvector that correspond to the smallest eigenvalue.

\begin{figure*}[!t]
    \centering
    \includegraphics[scale = 0.23]{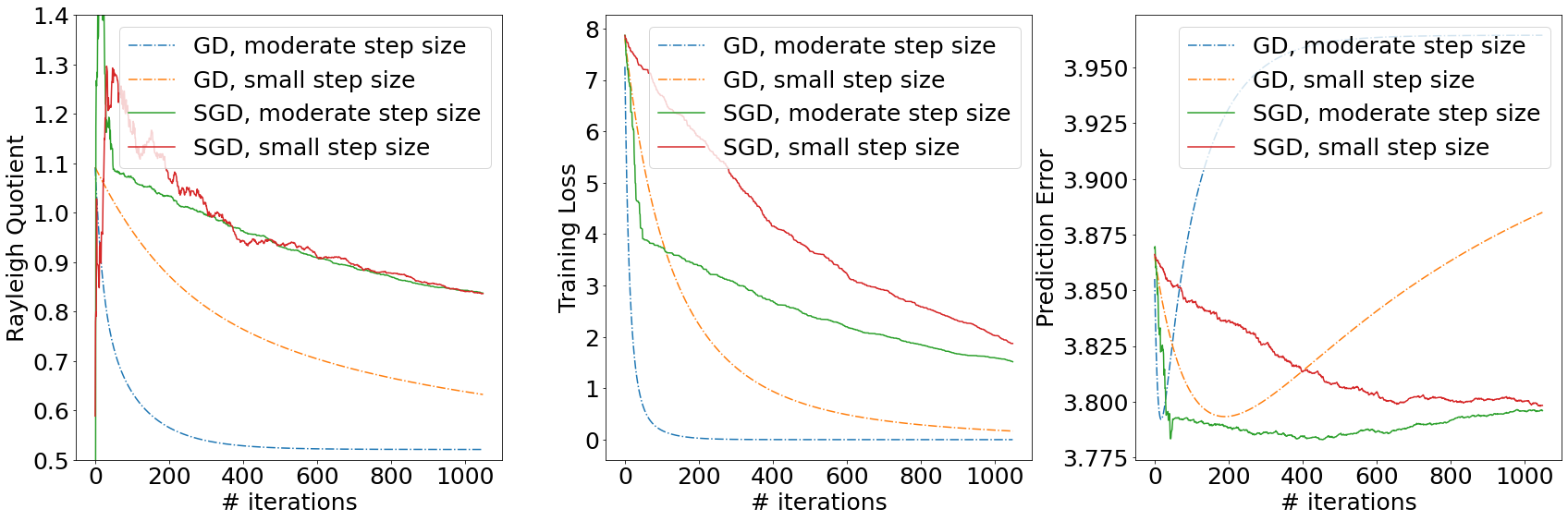}
    \caption{Kernel regression on synthetic data. The first plot shows directional bias by Rayleigh Quotient(RQ):= $\|K \bm{b}\|_2^2/\|\bm{b}\|_2^2$. The SGD indeed converges in the direction of a larger RQ, which matches our Theorems \ref{ch2:thm:informal1} and  \ref{ch2:thm:informal2}. In the third plot we show the prediction error of the solution paths, and the SGD does have lower prediction error than GD, even GD has smaller training loss than SGD. This supports Theorem \ref{ch2:thm:generalization}.}
    \label{ch2:fig:sim}
\end{figure*}

\subsection{Effect of directional bias}\label{ch2:sec:3.2}
In this subsection, the estimator that is biased towards the largest eigenvalue of the Hessian is shown to be the best for parameter estimation, see Theorem \ref{ch2:thm:quad_loss}. 
Later we define a realizable problem setting of kernel regression where the generalization error depends on the parameter estimation error, and in this way, the directional bias helps an SGD-based estimator to generalize.

\begin{theorem}\label{ch2:thm:quad_loss}
Consider minimizing the quadratic loss function 
\begin{equation*}
    L(\bm{w}) = \|A\bm{w} - \bm{y}\|_2^2. 
\end{equation*}
Assume there is a ground truth $\bm{w}^*$ such that $\bm{y} = A\bm{w}^*$. 
For a fixed level of the quadratic loss, the parameter estimation error $\|\bm{w} - \bm{w}^*\|_2^2$ has a lower bound: 
\begin{equation*}
    \forall \bm{w} \in \{\bm{w}: L(\bm{w}) = a\}: \|\bm{w} - \bm{w}^*\|_2^2 \geq a/\|A^T A\|_2. 
\end{equation*}
Moreover, the equality is obtained when $\bm{w}- \bm{w}^*$ is in the direction of the eigenvector that corresponds to the largest eigenvalue of matrix $A^T A$.
\end{theorem}
\begin{note}
Theorem \ref{ch2:thm:quad_loss} implies that the directional bias towards the largest eigenvalue is good for parameter estimation. 
As discussed in Remark \ref{ch2:note:07}, the SGD-based estimator is biased towards a larger eigenvalue compared to the GD-based estimator, by Theorem \ref{ch2:thm:quad_loss} the SGD estimator potentialy better estimates the true parameter and thus generalizes better, which we will formalize later.
\end{note}

Suppose $\exists f^* \in \bH$ such that $y = f^* (\bm{x})$. 
Consider the generalization error $L_D(f) := \| f - f^*\|_{\bH}^2$.
For an algorithm output $f^{\mbox{alg}}$, we decompose its generalization error as:
\begin{align*}
    &L_D(f^{\mbox{alg}}) - \inf_{f\in\bH} L_D(f)\\
    = &\underbrace{L_D(f^{\mbox{alg}}) - \inf_{f\in\bH_s} L_D (f)}_{:=\Delta(f^{\mbox{alg}}), \text{ estimation error}} + \underbrace{\inf_{f\in\bH_s} L_D(f) - \inf_{f\in\bH} L_D(f)}_{\text{approximation error}},
\end{align*}
where $\bH_s$ is the hypothesis class that the output of the algorithm is restricted to. By formulation \eqref{ch2:eq:reg_RKHS}, we have $\bH_s$:
$$
\bH_s = \{f\in \bH: f = \bm{\alpha}^T K(\cdot,X), \bm{\alpha}\in \bR^n \}. 
$$
We define the $a$-level set of training loss:
$$
\nu_a = \{f\in \bH_s: f = \bm{\alpha}^T K(\cdot,X), \frac{1}{2n}\|K\bm{\alpha} - \bm{y}\|_2^2 = a \},
$$ 
and denote $\Delta_a^* := \inf_{f\in\nu_a}\Delta(f)$.

Note that the approximation error can not be improved unless we change the hypothesis class, which is, changing the problem formulation in our case. 
So we just minimize the estimation error for estimators that are in the $a$-level set. One can check the estimation error is given by 
\begin{equation*}
   f\in \bH_s:  \Delta(f) = \bm{b}^TK\bm{b}, 
\end{equation*}
where $\bm{b} = \bm{\alpha} - \hat{\bm{\alpha}}$, seeing details in Appendix \ref{ch2:app:G2}. 
Similar to Theorem \ref{ch2:thm:quad_loss}, the estimation error is minimized when $\bm{b}$ is in the direction of the largest eigenvalue of $K$, so the directional bias towards a larger eigenvalue helps to generalize in kernel regression. We compare the estimation error of SGD and GD in following theorem. 

\begin{theorem}[Generalization performance]\label{ch2:thm:generalization}
Follow Theorems \ref{ch2:thm:informal1} and \ref{ch2:thm:informal2}, we have the following:
\begin{itemize}
    \item The output of SGD has $E[\Delta^{1/2}(f^{SGD})]\leq (1+4\epsilon)(\Delta_a^*)^{1/2}$, where $a$ is such that $E[\|K\bm{\alpha}^{SGD} - y\|_2]^2 = 2na$ and $\epsilon$ could be any positive small constant;
    \item The output of GD has $\Delta(f^{GD})\geq M\Delta_a^*$, where $a$ is the training loss of GD estimator, and $M = \frac{\gamma_1}{\gamma_n} (1-\epsilon') > 1$ is a large constant.
\end{itemize}
\end{theorem}
{
\begin{note}
This theorem indicates that $E[\Delta^{1/2}(f^{SGD})]\leq \Delta^{1/2}(f^{GD})$ when $1 + 4\epsilon \leq M^{1/2}$.
Taking $\epsilon < (\sqrt{\gamma_1/\gamma_n} - 1)/4$ in Theorem \ref{ch2:thm:informal1} and combining with Theorem \ref{ch2:thm:informal2} which states that $\epsilon'\stackrel{k\to\infty}{\longrightarrow} 0$, we will have $1 + 4\epsilon \leq M^{1/2}$ holds. 
In this way, $\Delta(f^{SGD}) < \Delta(f^{GD})$ with high probability. This finishes our claim that SGD generalizes better than GD.
\end{note}
}

\begin{figure*}
     \centering
     \begin{subfigure}{0.45\textwidth}
         \centering
         \includegraphics[scale = 0.2]{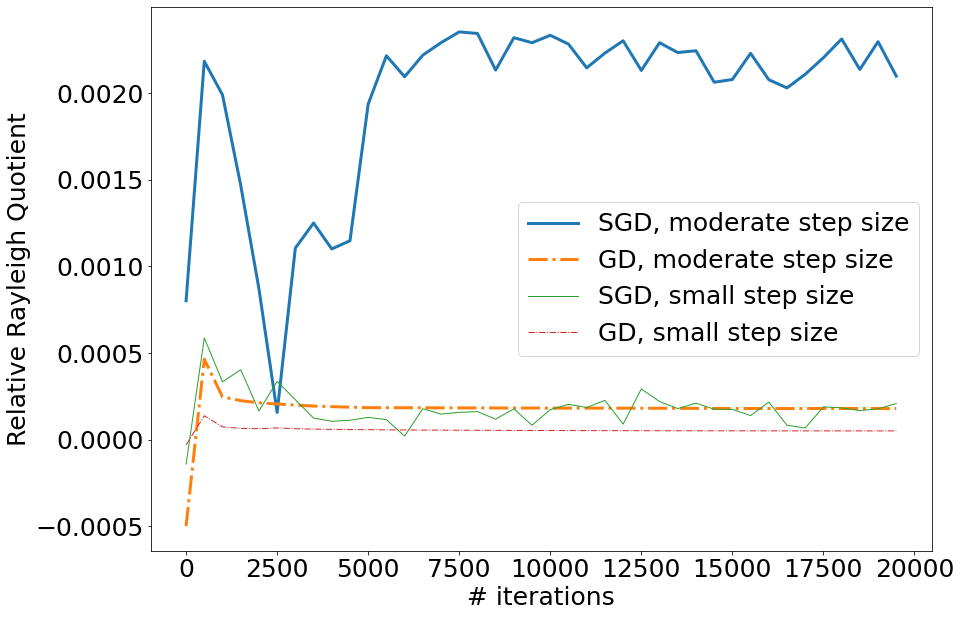}
        \caption{Relative Rayleigh Quotient.}
        \label{ch2:fig:RRQ_MNIST}
     \end{subfigure}
     \begin{subfigure}{0.45\textwidth}
         \centering
         \includegraphics[scale = 0.2]{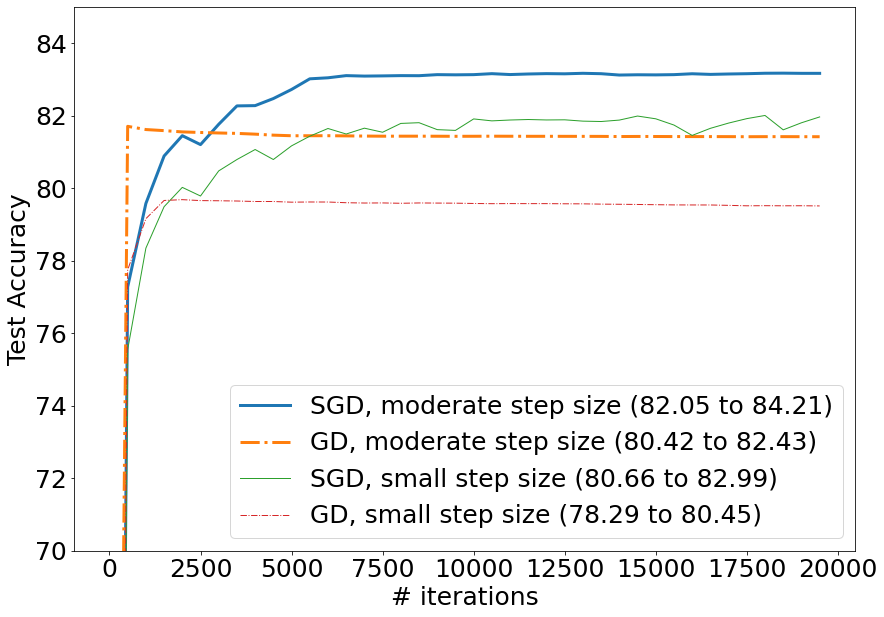}
        \caption{Test accuracy}
         \label{ch2:fig:acc_MNIST}
     \end{subfigure}
\caption{The experiment of a small ResNet on FashionMNIST. In (a), we follow \cite{wu2021direction} to use the Relative Rayleigh Quotient(RRQ) as the measurement of the convergence direction. 
The SGD with moderate step-size has higher RRQ than the GD with either moderate step-size or small step-size, which supports the theory in Theorems \ref{ch2:thm:informal1} and \ref{ch2:thm:informal2}. 
It is interesting to observe that SGD with a small step-size has a different directional bias compared with SGD with a moderate step-size, indicating that the directional bias studied in this work does not hold for a general SGD. 
In (b), we plot the testing accuracy from $20$ repetitions of experiments, the test accuracy (inside bracket) of SGD with moderate step-size is higher than the other cases, and we have Wilcoxon signed-rank test to confirm that the difference is significant at $0.01$ level. 
The test accuracy validates Theorem \ref{ch2:thm:generalization}. 
For more details of the experiments, the rank test, and more experiments, see Appendix \ref{ch2:app:H2}.
    }
    \label{ch2:fig:MNIST}
\end{figure*}

\section{Numeric Study}\label{ch2:sec:4}
\textbf{Simulation.} We simulate data from a nonlinear regression model with Gaussian additive noise as
$y_i = \sum_{j = 1}^{100}\sin(x_{i,j})+ \epsilon_i$, where $x_{i,j}\sim N(0,1)$ and $\epsilon_i\sim N(0,0.01)$. We fit kernel regression using the polynomial kernel $K(\bm{x}_1,\bm{x}_2) = (\langle \bm{x}_1,\bm{x}_2\rangle + .01)^2$ on $10$ training data and test the estimator on $5$ testing data. We run both SGD and GD for two step-size schemes: small step-size, and moderate annealing step-size. The results are in Fig. \ref{ch2:fig:sim}.

\textbf{Real data experiment.}
We run a 6-layer ResNet \cite{He_2016_CVPR} on FashionMNIST. The network structure is
\begin{align*}
    \text { Input } &\Rightarrow 7\times 7\text { Conv } \Rightarrow \text{BatchNorm}  \Rightarrow \text{ReLU}\\
    &\Rightarrow 3\times 3\text{ MaxPool }\Rightarrow \text{ResBlock1} \Rightarrow \text{ResBlock2}\\ &\Rightarrow \text{Global AvePool} \Rightarrow \text{FC} \Rightarrow \text{output}.
\end{align*}
We run SGD and GD for two step-size schemes, similar to our simulation. There are $1,500$ training data and $10,000$ testing data in our experiment. The result is in Fig. \ref{ch2:fig:MNIST}.

\begin{note}
The purposes of experiment using a Neural Network (Fig. \ref{ch2:fig:MNIST}) are: first, the Neural Network results support our finding on kernel regression, since Neural Network is related to kernel regression through NTK theory \cite{Arthur2018}; second, our experiment indicates that our result may be empirically true for the more general deep learning framework \cite{belkin2018understand}.
\end{note}



\section{Discussion and Further Work}\label{ch2:sec:5}
We advance one more step towards understanding the directional bias of SGD in kernel learning.
We discuss some implications of our results. 
 
\textbf{Implication to the SGD scheme}: 
Our result shows the directional bias holds to SGD with annealing step-size. 
Specifically, the first stage of SGD with moderate step-size should run long enough, then in the second stage by decreasing step-size we have the directional bias towards the largest eigenvalue of the Hessian, which helps in obtaining a better generalization error bound. 
This explains a technique for tuning the learning rate that people adopt in practice: starting with a large step-size, running long enough until the error plateaus, then decreasing the step-size \cite{He_2016_CVPR}. 
Although this technique is always used for speed convergence, we show that it also helps in predictive power, which becomes even better.
 
\textbf{Implication to deep learning}: 
Our assumption in the analysis implies certain structures for the deep learning models. 
Per our discussion in Remark \ref{ch2:note:diag_dom_note2}, our assumption holds when the feature space is high dimensional and/or when features are possibly sparse. 
This matches the deep learning scenario where we have a highly overparameterized model and when the trained parameter estimator becomes sparse. 
In addition, considering that some deep learning tasks can be approximated by kernel learning \cite{Arthur2018}, our results help to explain why the SGD-based estimator can perform better in an overparameterized deep learning setting.

Just as stated in \cite{belkin2018understand}, to understand deep learning one needs to understand kernel learning. 
This work improves our understanding in kernel learning, and can possibly lead to deep understanding of practices in deep learning. 

\printbibliography

\appendix

\section{Background on RKHS}\label{ch2:app:RKHS}
This section details the background on RKHS in two subsections. The first subsection includes notations, theorems, and an example of RKHS, the second section reduces the kernel regression in RKHS from infinite dimension to finite dimension, which gives our objective function \eqref{ch2:eq:reg_RKHS} . 

\subsection{Nonparametric model in RKHS}
In this subsection, we give the definition and notations for our model in RKHS, and its associated norms, basis, etc. The definitions are similar to those in \cite{raskutti2012minimax}.\\
\\
Given $n$ data pairs $\{\bm{x}_i,y_i\}_{i=1}^n$, where $\bm{x}_i\in\bX \subset\bR^p$ and $y_i\in \bR$, assume that $y_i$s are associated with $\bm{x}_i$s through $f(\bm{x}_i)$, where $f(\cdot)$ is some unknown function in the reproducing kernel Hilbert space (RKHS) of functions $\bX\to\bR$, our goal is to estimate the function $f(\cdot)$ from the data. \\
\\
Denote the RKHS where $f$ lives as $\bH$, with reproducing kernel $K: \bX \times \bX \to \bR_+$ (which is known to us). And we associate the functions in $\bH$ with probability measure $\mathbbm{Q}$, assume w.l.o.g. that $\int_{\bR^p} f(\bm{a}) d\mathbbm{Q}(\bm{a}) = 0$. By Mercer's theorem, $K$ has eigen-expansion:
\begin{equation}\nonumber
    K(\bm{a},\bm{b}) = \sum_{j=1}^{\infty} \gamma_{j}\phi_{j}(\bm{a}) \phi_{j}(\bm{b})
\end{equation}
Where $\{\phi_{j}\}_{j=1}^\infty$ are orthonormal basis in $\mathbbm{L}^2(\mathbbm{Q})$, w.r.t. the usual inner product in $\mathbbm{L}^2(\mathbbm{Q})$ as
$$\langle g(\cdot), h(\cdot) \rangle_{\mathbbm{L}^2(\mathbbm{Q})} = \int_{\bX} g(\bm{a}) h(\bm{a})d\mathbbm{Q}(\bm{a})$$
Now for any $f\in\bH$, we can expand 
$f(\cdot) = \sum_{j=1}^{\infty} c_{j} \phi_{j}(\cdot)$, where $c_{j} = \langle f(\cdot), \phi_{j}(\cdot) \rangle_{\mathbbm{L}^2(\mathbbm{Q})}$. And for $f(\cdot) = \sum_{j=1}^{\infty} c_{j} \phi_{j}(\cdot)$, $g(\cdot) = \sum_{j=1}^{\infty} c'_{j} \phi_{j}(\cdot)$, by Parseval’s theorem

$$\langle f(\cdot), g(\cdot) \rangle_{\mathbbm{L}^2(\mathbbm{Q})} = \sum_{j=1}^{\infty} c_j c'_{j}$$
And we have another inner product that is defined for RKHS $\bH$ as
$$\langle f(\cdot), g(\cdot) \rangle_{\bH} = \sum_{j=1}^{\infty} \frac{c_{j} c'_{j}}{\gamma_{j}}$$
The reproducing property of RKHS says that $\forall f\in \bH$, we have
\begin{align}\nonumber
\langle f(\cdot),K(\cdot,\bm{x})  \rangle_{\bH} = f(\bm{x})
\end{align}

\textbf{Cubic Splines Formulate a RKHS}. We go over an example of RKHS for better understanding. Consider the cubic spline of one dimension, we can show that the space of cubic splines is a RKHS. One can also find the cubic spline example in \cite{hazimeh2021grouped}. For more details on the relationship between polynomial smoothing splines and RKHS, one can check Section 1.2 of \cite{wahba1990spline} .

Assume w.l.o.g. that $x_i \in \bX= [0,1]\subset\bR$. The cubic spline $f$ on $\bX$ is continuous, has a continuous first-order derivative and square integrable second order derivative. By Taylor's theorem with remainder, we have
 \begin{align*}
 f(t) &= f(0) + t f'(0) + \int_{0}^t (t-u) f''(u) du\\
 &= f(0) + t f'(0) + \int_{0}^1 (t-u)_+ f''(u) du
 \end{align*}
 where $(t-u)_+ = \max\{0,t-u\}$. Let $\mathcal{B}$ be the set of cubic splines $f$ on $[0,1]$ that satisfies the boundary condition $f(0) = f'(0) = 0$, then for $f\in\mathcal{B}$
 \begin{align*}
 f(t) = \int_{0}^1 (t-u)_+ f''(u) du
 \end{align*}
 Let $G(t,u) = (t-u)_+$, then we claim that $\mathcal{B}$ is RKHS with reproducing kernel
 \begin{align}
K(s,t) = \int_{0}^1G(s,u)G(t,u) du\nonumber
 \end{align}
 and inner product
 \begin{align}
 \langle f,g \rangle_{\mathcal{B}} = \int_{0}^1 f''(u)g''(u)du\nonumber
 \end{align}
as one can check the reproducing property
\begin{align}\nonumber
 \langle f(\cdot),K(\cdot,t) \rangle_{\mathcal{B}}&=\int_{0}^1\frac{\partial^2 K(u,t)}{\partial u^2} f''(u)du\\
 &=\int_{0}^1(t-u)_+f''(u)du = f(t)\nonumber
\end{align}
\subsection{Optimization problem considered}
This subsection gives problem formulation of kernel regression. Given data pairs $\{\bm{x}_i,y_i\}_{i=1}^n$ and RKHS $\bH$, consider a loss function $\ell$ which is selected according to how $y$ is connected with $f(\bm{x})$, we may estimate the model by 
\begin{align}
\label{ch2:eq:obj}
    &\min_{f\in \bH} \frac{1}{n}\sum_{i=1}^n\ell(y_i,f(\bm{x}_i))\\
    =&\min_{c_{j}}\frac{1}{n} \sum_{i=1}^n\ell(y_i,\sum_{j=1}^{\infty} c_{j} \phi_{j}(\bm{x}_{i}))\nonumber
\end{align}
Example for $\ell$ includes
\begin{itemize}
    \item Squared error loss $\ell(y,f) = (y-f)^2$, which is usually used in regression;
    \item 0-1 loss $\ell(y,f) = \mathbf{1}(y*f>0)$, for binary classification;
    \item Logistic loss $\ell_(y,f) = \log(1+ \exp(-y*f))$, also a loss function for classification, can be considered as a surrogate function of 0-1 loss, and is the same as negative log likelihood function in logistic regression.
\end{itemize}
Let us come back to the nonparametric model part, to control the model smoothness, the usual practice is to add a penalty to objective \eqref{ch2:eq:obj}, result in 
\begin{align*}
    \begin{split}
    &\min_{f\in \bH} \sum_{i=1}^n\ell(y_i,f(\bm{x}_i)) + \lambda \text{pen}(f)
    \end{split}
\end{align*}
A popular choice of $\text{pen}(f)$ is $\|f\|_{\bH}^2$, or any strictly increasing function of $\|f\|_{\bH}^2$. Such method is explicitly controlling the model smoothness, and by Representer Theorem, it has solution of the form
\begin{align}\label{ch2:eq:rep}
    {f}(\cdot) =\sum_{i=1}^n \alpha_{i}K(\cdot,\bm{x}_{i})
\end{align}
Plug equation \eqref{ch2:eq:rep} into objective \eqref{ch2:eq:obj}, we have the problem becomes
\begin{equation*}
    \min_{\alpha_{i'}} \frac{1}{n}\sum_{i=1}^n\ell(y_i,\sum_{i'=1}^n \alpha_{i'}K(\bm{x}_{i},\bm{x}_{i'}))
\end{equation*}
Which gives the formulation \eqref{ch2:eq:reg_RKHS} under loss function $\ell(y,f(\bm{x})) = (y-f(\bm{x}))^2/2$.

\section{Diagonal Dominance of Some Popular Kernels}\label{ch2:app:diagonal_dominance}
In this section, we justify Assumption \ref{ch2:ass:1} by figuring out a problem setting where some popular kernels give a diagonal dominant Gram matrix. For simplicity, we assume the following data distribution throughout this section: 
\begin{subequations}
\label{ch2:eq:assumptions_a}
\begin{align}
&\text{$\mathbf{x}_i \in R^d, i=1,\ldots,n$, are normalized such that $\|\bm{x}_i\|_2^2 = 1$};\\
&\text{The direction of $\bm{x}_i$s are i.i.d. uniformly distributed on the unit sphere $S^{d-1}$};\\
&\text{$d\gg n$ (overparameterized setting).}
\end{align}
\end{subequations}
Given assumption set (\ref{ch2:eq:assumptions_a}), we can bound the inner product of data $\langle \bm{x}_i,\bm{x}_j\rangle$ with high probability as follows:
\begin{lemma}[Lemma 1 in \cite{wu2021direction}]\label{ch2:lem:A1}
Under assumption set \eqref{ch2:eq:assumptions_a}, let $d \geq 4 \log(2n^2/\delta)$ for some $\delta\in(0, 1)$. Then with probability at least $1 - \delta$, we have
$$ |\langle \bm{x}_i,\bm{x}_j\rangle| <\tilde{\tau} := \Tilde{\mathcal{O}}(\frac{1}{\sqrt{d}}),\forall i \neq j$$
\end{lemma}
\begin{proof}
See proof of Lemma 1 in \cite{wu2021direction}.
\end{proof}

The bound on the inner product $\langle \bm{x}_i,\bm{x}_j\rangle$ induces bound on $K(\bm{x}_i,\bm{x}_j)$ for some popular kernels. We show the diagonal dominance for two groups of kernels in the following propositions, and list some examples for kernels in each group.

\begin{proposition}[Inner product kernel]\label{ch2:prop:innerproduct}
The inner product kernel is defined as a smooth transformation of inner product. We can write it as:
$$K(\bm{x}_i,\bm{x}_j) = g(\langle \bm{x}_i, \bm{x}_j \rangle)$$
Assume assumptions \eqref{ch2:eq:assumptions_a} hold, and assume the function $g: [-1,1]\to R$ satisfies:
\begin{subequations}
\label{ch2:eq:assumptions_b}
\begin{align}
    &g \text{ is convex};\label{ch2:eq:a1}\\
    &\text{$g$ is $L-$smooth, that is, $\nabla g$ is $L-$Lipschitz continuous};\label{ch2:eq:a2}\\
    &|g(0)|\leq c\tilde{\tau}\text{ for a constant $c$}, g'(0) \geq 0.\label{ch2:eq:a3}
\end{align}
\end{subequations}
we will have with probability $1-\delta$
\begin{align}
    &|K_{i,j}| \leq (c+g'(0))\tilde{\tau} + \frac{L}{2}\tilde{\tau}^2 \text{ for } i\neq j
\end{align}
where $\delta$ and $\Tilde{\tau}$ the same as in Lemma \ref{ch2:lem:A1}.When $g'(0)\tilde{\tau} + \frac{L}{2}\tilde{\tau}^2\ll g(1)$ for a small enough $\Tilde{\tau}$, the Gram matrix is diagonal dominant.
\end{proposition}
\begin{proof}
We have the following with probability at least $1- \delta$ by Lemma \ref{ch2:lem:A1}. For any off-diagonal elements of $K$:
\begin{align*}
\begin{split}
    K_{i,j} &= g( \langle \bm{x}_i, \bm{x}_j \rangle)\\
    &\geq g(0) +  g'(0) \langle \bm{x}_i, \bm{x}_j \rangle \\
    &\geq - (g'(0)+c) \tilde{\tau}
\end{split}
\end{align*}
and
\begin{align*}
\begin{split}
    K_{i,j} &= g( \langle \bm{x}_i, \bm{x}_j \rangle)\\
    &\leq g(0) + g'(0) \langle \bm{x}_i, \bm{x}_j \rangle + \frac{L}{2}\langle \bm{x}_i, \bm{x}_j \rangle^2 \\
    &\leq (c+ g'(0)) \tilde{\tau} + \frac{L}{2} \tilde{\tau}^2
\end{split}
\end{align*}
Thus $|K_{i,j}| \leq (c + g'(0)) \tilde{\tau} + \frac{L}{2} \tilde{\tau}^2$.
\end{proof}
\begin{note}
We list some examples of inner product kernels that give diagonal dominant kernel matrices:
\begin{itemize}
    \item \textbf{Bilinear Kernel:} $K(\bm{x},\bm{x}') = \langle \bm{x},\bm{x}' \rangle$, then
    $$|K(\bm{x}_i,\bm{x}_j)|\leq \tilde{\tau}\ll K(\bm{x}_n,\bm{x}_n) = 1.$$
    
    \item \textbf{Polynomial Kernel:} $K(\bm{x},\bm{x}') = (\langle \bm{x},\bm{x}' \rangle + c)^m$ for $m\in \mathbbm{N}$ and $c\sim \mathcal{O}(\tilde{\tau})$, then by Proposition \ref{ch2:prop:innerproduct}
        $$|K(\bm{x}_i,\bm{x}_j)| \leq (1+m)\tilde{\tau} + \frac{m*\exp((m-1)\tilde{\tau})}{2}\tilde{\tau}^2$$
   when $(1+m)\tilde{\tau} + \frac{m*\exp((m-1)\tilde{\tau})}{2}\tilde{\tau}^2 \ll (1 + c)^m$, we have diagonal dominant Gram matrix.
    
    \item \textbf{Hyperbolic Tangent Kernel (Sigmoid Kernel):} $K(\bm{x},\bm{x}') = \tanh (\alpha\langle\bm{x},\bm{x}'\rangle + c)$, where $\alpha > 0, c\geq 0$. Note that Hyperbolic Tangent Kernel does not satisfy all the assumptions in Proposition \ref{ch2:prop:innerproduct}, one can still calculate 
    $$|K(\bm{x}_i,\bm{x}_j)| \leq \tanh (\alpha\tilde{\tau} + c) $$
    and 
    $$K(\bm{x}_k,\bm{x}_k) = \tanh(\alpha + c)$$ 
    When $\tanh (\alpha\tilde{\tau} + c)\ll \tanh (\alpha + c)$ (which is the case when $\alpha$ is large, and $c, \tilde{\tau}$ are small enough), we have $|K(\bm{x}_i,\bm{x}_j)|\ll K(\bm{x}_n,\bm{x}_n)$ and the Gram matrix is diagonal dominant.
\end{itemize}
\end{note}

\begin{proposition}[Radial Basis Function (RBF) kernel]
Radial Basis Function kernel depends on two data points through their distance, which is of following form
$$K(\bm{x}_i,\bm{x}_j) = \exp(-\gamma \| \bm{x}_i - \bm{x}_j\|_2^2), \gamma > 0$$
Assume assumptions \eqref{ch2:eq:assumptions_a} hold, when $\gamma = -c_0\log (\tilde{\tau})$ for a constant $c_0$, we have with probability $1-\delta$
\begin{align*}
    &|K_{i,j}| \leq \tilde{\tau}^{2c_0(1 - \tilde{\tau})}\ll K_{n,n} = 1,  \text{ for } i\neq j
\end{align*}
where $\delta$ and $\Tilde{\tau}$ the same as in Lemma \ref{ch2:lem:A1}. That is, the Gram matrix is diagonal dominant.
\end{proposition}
\begin{proof}
Bound off-diagonal terms of $K$ by Lemma \ref{ch2:lem:A1}:
\begin{align*}
\begin{split}
    K_{i,j} &= \exp(-\gamma\|\bm{x}_i- \bm{x}_j\|_2^2)\\
    &\leq \exp( 2\gamma\tilde{\tau} - 2\gamma)\\
    &\ = \tilde{\tau}^{2c_0(1-\tilde{\tau})}
\end{split}
\end{align*}
\end{proof}
\begin{note}
We note some popular kernels that are related to Radial Basis Function kernel, and show that they lead to diagonal dominance:
\begin{itemize}
    \item \textbf{Gaussian Kernel:} $K(\bm{x},\bm{x}') = \exp(-\frac{\|\bm{x} - \bm{x}'\|_2^2}{2\sigma^2}) $. One can see that Gaussian Kernel is reparameterizing RBF Kernel by $\gamma = 1/(2\sigma^2)$. Thus Gaussian Gram matrix is diagonal dominant when $\sigma^2 \sim \mathcal{O}( -\frac{1}{\log(\tilde{\tau})})$.
    
    \item \textbf{Laplace Kernel:} $K(\bm{x},\bm{x}') = \exp(-\frac{\|\bm{x} - \bm{x}'\|_2}{\sigma})  $ for $\sigma > 0$. The Laplace Kernel is very similar to Gaussian Kernel, and one can check by similar steps that when $\sigma \sim \mathcal{O}( -\frac{1}{\log(\tilde{\tau})})$, Laplace Gram matrix is diagonal dominant. 
\end{itemize}
\end{note}

\section{Lemmas}
This section includes two useful lemmas for characterizing the eigenvalues of a symmetric matrix.
\begin{lemma}[Gershgorin circle theorem, restated for symmetric matrix]\label{ch2:lem:B1}
Let $A\in \bR^{n\times n}$ be a symmetric matrix. Let $A_{ij}$ be the entry in the $i-$th row and the $j-$th column. Let
$$R_i(A) :=
\sum_{j\neq i} |A_{ij}|, i = 1,\ldots,n$$
Consider n Gershgorin discs
$$D_i(A) := \{z \in \bR, |z - A_{ii}|\leq R_i(A)\}, i=1,
\ldots,n$$
The eigenvalues of A are in the union of Gershgorin discs 
$$G(A) := \cup_{i=1}^n D_i(A)$$
Furthermore, if the union of $k$ of the $n$ discs that comprise $G(A)$ forms a set $G_k(A)$ that is disjoint
from the remaining $n-k$ discs, then $G_k(A)$ contains exactly $k$ eigenvalues of $A$, counted according
to their algebraic multiplicities.
\end{lemma}
\begin{proof}
See \cite{horn2012matrix}, Chap 6.1, Theorem 6.1.1.
\end{proof}

\begin{lemma}[Cauchy interlacing theorem, restated for symmetric matrix]\label{ch2:lem:B2}
Let $B \in \bR ^{m\times m}$ be a symmetric matrix, let $\bm{y} \in\bR^n$ and $a \in\bR$, and let $A = \begin{bmatrix}B & \bm{y}\\\bm{y}^T & a\end{bmatrix}$. Then
\begin{align*}
    \lambda_1(A) \geq \lambda_1(B) \geq \lambda_2(A) \geq \ldots \geq \lambda_m(A) \geq  \lambda_m(B) \geq \lambda_{m+1}(A).
\end{align*}
\end{lemma}
\begin{proof}
See \cite{horn2012matrix}, Chap 4.3, Theorem 4.3.17.
\end{proof}

\section{Spectrum of Gram matrix}
This section analyzes the eigen structure of the Gram matrix.
\begin{lemma}[Characterizing $K^2$]\label{ch2:lem:C1} Under Assumption \ref{ch2:ass:1}, we have
\begin{align}
    \langle \bm{K}_i, \bm{K}_i \rangle \in [\lambda_i^2,\lambda_i^2 + (n-1)  \tau^2 ]\label{ch2:eq:lemC1eq1}\\
    |\langle \bm{K}_i, \bm{K}_j \rangle| \leq [2\lambda_1 + (n-2)\tau]\tau, i\neq j\label{ch2:eq:lemC1eq2}
\end{align}
\end{lemma}
\begin{proof}
For $\langle \bm{K}_i, \bm{K}_i\rangle$
\begin{align*}
    \langle \bm{K}_i, \bm{K}_i\rangle &= K_{i,i}^2 + \sum_{l \neq i} K_{l,i}^2 \\
    &\in [\lambda_i^2,\lambda_i^2 + (n-1)  \tau^2 ]
\end{align*}
And for $\langle \bm{K}_i, \bm{K}_j\rangle, i\neq j$
\begin{align*}
    \begin{split}
        |\langle \bm{K}_i, \bm{K}_j\rangle| &= |\sum_{l = 1}^n K_{l,i} K_{l,j}|\\
        &=|K_{i,i} K_{i,j} + K_{j,i} K_{j,j} + \sum_{l \neq i,j} K_{l,i} K_{l,j}|\\
        &\leq |K_{i,i} K_{i,j}| + |K_{j,i} K_{j,j}| + \sum_{l \neq i,j} |K_{l,i} K_{l,j}|\\
        &\leq [\lambda_i + \lambda_j]\tau + (n-2) \tau^2\\
        &\leq [2\lambda_1 + (n-2)\tau]\tau
    \end{split}
\end{align*}
\end{proof}

\begin{lemma}[Eigenvalue of $K$]\label{ch2:lem:C2}
Under Assumption \ref{ch2:ass:1}, we have
\begin{align}
    &\gamma_1 \leq \lambda_1 + n\tau \label{ch2:eq:lemB2_eq1}\\
    &\gamma_n \geq \lambda_{n} - n\tau
\end{align}
If we further assume $\lambda_j + n\tau < \lambda_{j-1} - n\tau$, we will have
\begin{align*}
    \gamma_{j-1} \geq \lambda_{j-1} - n\tau > \lambda_j + n\tau \geq \gamma_j.
\end{align*}
\end{lemma}
\begin{proof}
Use Gershgorin circle theorem, calculate
\begin{align*}
    R_i(K) = \sum_{j\neq i}|K_{i,j}|\leq n\tau
\end{align*}
then 
\begin{align*}
    D_i(K) \subset [\lambda_i- n\tau, \lambda_i + n\tau].
\end{align*}
By Gershgorin circle theorem, the lemma claim holds. 
\end{proof}

\begin{lemma}[Characterize $P_1 K$ and $P_{-1}K$]\label{ch2:lem:C3}
Recall our definition: $P_{-1}$ is the projection on column space of  $K_{-1} = [\bm{K}_2,\bm{K}_3,\ldots,\bm{K}_n]$, and $P_1 = I - P_{-1}$. We claim the following hold
\begin{align}
    P_1 K &= [P_1 \bm{K}_1,\bm{0},\ldots,\bm{0}] \label{ch2:eq:lemB3_eq1}\\
    P_{-1} K &= [P_{-1} \bm{K}_1,\bm{K}_2,\ldots,\bm{K}_n] \label{ch2:eq:lemB3_eq2}
\end{align}

Assume $\tau$ is small enough that $n\tau \leq \mathcal{O}(1)$, $\lambda_{n} - n\tau \geq c_1 >0$ and $\lambda_{1} + n\tau \leq c_2$, let $c_3 := \frac{c_2}{c_1^2} $, then we have the following:
\begin{align}
    &\|P_{-1} \bm{K}_1\|_2\in \left[0, c_3  (2\lambda_1 + (n-2)\tau) \sqrt{n}\tau\right]\label{ch2:eq:lemB3_eq3}\\
    &\|P_1 \bm{K}_1\|_2 \in\left[\sqrt{\lambda_1^2 - c_3^2 [2\lambda_1 + (n-2)\tau]^2 n\tau^2}, \lambda_1 + \sqrt{n} \tau\right]\label{ch2:eq:lemB3_eq4}.
\end{align}
\end{lemma}
\begin{proof}
For $i\neq 1$, we calculate 
\begin{align*}
    P_{-1} \bm{K}_i &= \bm{K}_i\\
    P_1 \bm{K}_i &= \bm{K}_i - P_{-1} \bm{K}_i = \bm{0}
\end{align*}
thus we have equations \eqref{ch2:eq:lemB3_eq1} and \eqref{ch2:eq:lemB3_eq2}.
\\
\\
For $\|P_{-1}\bm{K}_1\|_2$:
\begin{align*}
    \begin{split}
        &\|P_{-1}\bm{K}_1\|_2 \\
        =&\|K_{-1}(K_{-1}^T K_{-1})^{-1}K_{-1}^T \bm{K}_1\|_2\\
        \leq & \|K_{-1}(K_{-1}^T K_{-1})^{-1}\|_2 \|K_{-1}^T \bm{K}_1\|_2
    \end{split} 
\end{align*}
where 
\begin{align*}
    \begin{split}
        &\|K_{-1}^T \bm{K}_1\|_2^2 \\
        =& \sum_{i=2}^n (\bm{K}_i^T \bm{K}_1)^2\\
        \stackrel{\eqref{ch2:eq:lemC1eq2}}{\leq}& (n-1) [2\lambda_1 + (n-2)\tau]^2 \tau^2\\
        \leq & [2\lambda_1 + (n-2)\tau]^2 n\tau^2
    \end{split}
\end{align*}
and $K_{-1}^T K_{-1} $ has all eigenvalues in $[\gamma_n^2, \gamma_1^2]$ by Cauchy interlacing theorem (Lemma \ref{ch2:lem:B2}), that is, all singular values of $K_{-1}$ are in $[c_1,c_2]$ by our assumption. Then 
\begin{align*}
\|K_{-1}(K_{-1}^T K_{-1})^{-1}\|_2 \leq \frac{c_2}{c_1^2} := c_3 
\end{align*}
So we have
\begin{align*}
    \|P_{-1}\bm{K}_1\|_2 &\leq \|K_{-1}(K_{-1}^T K_{-1})^{-1}\|_2 \|K_{-1}^T \bm{K}_1\|_2\\
    &\leq c_3  [2\lambda_1 + (n-2)\tau] \sqrt{n}\tau.
\end{align*}
For $\|P_1\bm{K}_1\|_2$:
\begin{align*}
    &\|P_1 \bm{K}_1\|_2 \leq \| \bm{K}_1\|_2\\
    \leq& (\lambda_1^2 + (n-1) \tau^2)^{.5}\\
    \leq& \lambda_1 + \sqrt{n} \tau 
\end{align*}
and 
\begin{align*}
    &\|P_1 \bm{K}_1\|_2^2 \\
    =& \| \bm{K}_1\|_2^2 -  \|P_{-1}\bm{K}_1\|_2^2\\
    \geq & \lambda_1^2 - c_3^2 [2\lambda_1 + (n-2)\tau]^2 n\tau^2.
\end{align*}
\end{proof}

\begin{lemma}[Spectrum of $H_{-1} := P_{-1} K K^T P_{-1}$]\label{ch2:lem:C4} 
Assume $$c_3^2[2\lambda_1 + (n-2)\tau]^2 n \tau^2 + 2[2\lambda_1 + (n-2)\tau]n\tau \leq \lambda_n^2 $$ 
We have the following:
\begin{itemize}
    \item $0$ is an eigenvalue of $H_{-1}$, corresponding eigenspace is the column space of $P_1$;
    \item Restricted to the column space of $P_{-1}$, the eigenvalues of $H_{-1}$ are all in the interval:
     $$\left(\lambda_n^2 - [2\lambda_1 + (n-2)\tau]n\tau,\lambda_2^2 + [2\lambda_1 + (n-1)\tau]n\tau\right).$$
\end{itemize}
\end{lemma}
\begin{proof}
The first claim is by construction of $P_1$ and $P_{-1}$.\\
\\
For the second claim, note that $H_{-1}$ has the same eigenvalues as 
\begin{equation*}
    H_{-1}' = (P_{-1} K)^T P_{-1}K
\end{equation*}
Now the diagonal entries of $H_{-1}'$ are:
\begin{align*}
    (H_{-1}')_{ii} = \|P_{-1} \bm{K}_i\|_2^2 = \left\{\begin{array}{ll}
        \|P_{-1}\bm{K}_1\|_2^2\leq c_3^2 [2\lambda_1 + (n-2)\tau]^2 n\tau^2 &, i = 1 \\
         \|\bm{K}_i\|_2^2 \in [\lambda_i^2, \lambda_i^2 + (n-1)\tau^2]&, i\neq 1
    \end{array}\right.
\end{align*}
And the off-diagonal entries of $H_{-1}'$ are:
\begin{align*}
    |(H_{-1}')_{ij}| = |\langle P_{-1}\bm{K}_i,P_{-1}\bm{K}_j\rangle| = |\langle \bm{K}_i,\bm{K}_j\rangle|\leq [2\lambda_1 + (n-2)\tau]\tau
\end{align*}
To use Gershgorin circle theorem, calculate
\begin{align*}
    R_{i}(H_{-1}') = \sum_{j\neq i} |(H_{-1}')_{ij}| < [2\lambda_1 + (n-2)\tau]n\tau
\end{align*}
Thus the Gershgorin discs:
\begin{align*}
    &D_1(H_{-1}') \in (\|P_{-1}\bm{K}_1\|_2^2 - [2\lambda_1 + (n-2)\tau]n\tau,\|P_{-1}\bm{K}_1\|_2^2 + [2\lambda_1 + (n-2)\tau]n\tau)\\
    &D_i(H_{-1}') \in (\|\bm{K}_i\|_2^2 - [2\lambda_1 + (n-2)\tau]n\tau,\|\bm{K}_i\|_2^2 + [2\lambda_1 + (n-2)\tau]n\tau)
\end{align*}
when $c_3^2[2\lambda_1 + (n-2)\tau]^2 n \tau^2 + [2\lambda_1 + (n-2)\tau]n\tau \leq \lambda_n^2 - [2\lambda_1 + (n-2)\tau]n\tau$, the first Gershgorin discs does not intersect with the others, so we have $n-1$ nonzero eigenvalues in 
$$\cup_{i=2}^n D_i(H_{-1}')\subset \left(\lambda_n^2 - [2\lambda_1 + (n-2)\tau]n\tau,\lambda_2^2 + [2\lambda_1 + (n-1)\tau]n\tau\right).$$
\end{proof}

\section{Directional bias of SGD with moderate step size}\label{ch2:app:E}
This section gives formal proof of Theorem \ref{ch2:thm:informal1} and specifies the constants. The proof is done in four steps: Lemma \ref{ch2:lem:E1} analyzes one update of SGD; Lemma \ref{ch2:lem:E2} uses Lemma \ref{ch2:lem:E1} to bound the first stage updates of SGD with moderate step size; Lemma \ref{ch2:lem:E3} again uses Lemma \ref{ch2:lem:E1}, and bounds the second stage updates of SGD with small step size; finally, Theorem \ref{ch2:thm:E4} combines Lemma \ref{ch2:lem:E2} and Lemma \ref{ch2:lem:E3} to formalize the directional bias of SGD, it is the same as Theorem \ref{ch2:thm:informal1}, but restated using the constants defined therein.

\begin{lemma}[One step update of SGD]\label{ch2:lem:E1}
Under Assumption \ref{ch2:ass:1}, denote $A_t := E[\|P_1 \bm{b}_t\|_2]$, $B_t := E[\|P_{-1} \bm{b}_t\|_2]$, fix a constant $c_4 \geq (\lambda_1 + \sqrt{n} \tau)(2\lambda_1 + (n-2)\tau)c_3$, then we have:
\begin{align}
    &A_{t + 1} \leq q_1(\eta) A_t + \xi(\eta) B_t\label{ch2:eq:lemc1_1}\\
    &A_{t + 1} \geq q_1(\eta) A_t - \xi(\eta) B_t\label{ch2:eq:lemc1_2}\\
    &B_{t + 1} \leq q_{-1}(\eta) B_t + \xi(\eta) A_t\label{ch2:eq:lemc1_3}
\end{align}
where
\begin{align*}
    &q_1(\eta) = \frac{n-1}{n} + \frac{1}{n}|1 - \eta\|P_1 \bm{K}_1\|_2^2|\\
    &q_{-1}(\eta) = \sqrt{1 + \frac{\| K P_{-1} \bm{b}_t\|_2^2}{n\|P_{-1} \bm{b}_t\|_2^2} \left[\eta^2(\lambda_2^2 + (n-1)\tau^2) - 2\eta\right]}\\
    &\xi(\eta) =c_4 \eta  n^{-1/2} \tau.
\end{align*}
\end{lemma}
\begin{proof}
One step of SGD update is:
\begin{equation*}
    \bm{b}_{t+1} = \bm{b}_{t} - \eta \bm{K}_i \bm{K}_i^T \bm{b}_t = [I-\eta \bm{K}_i \bm{K}_i^T]\bm{b}_t
\end{equation*}
where $i$ is uniformly random sample from $[1,\ldots,n]$.\\
\\
For inequalities \eqref{ch2:eq:lemc1_1} and \eqref{ch2:eq:lemc1_2}, check  
\begin{align*}
    P_1 \bm{b}_{t+1} &= P_1[I-\eta \bm{K}_i \bm{K}_i^T]\bm{b}_t\\
    &= P_1 \bm{b}_t - \eta P_1 \bm{K}_i \bm{K}_i^T (P_1 + P_{-1}) \bm{b}_t\\
    &= [I - \eta P_1 \bm{K}_i \bm{K}_i^T P_1] P_1 \bm{b}_t - \eta [P_1 \bm{K}_i \bm{K}_i^T P_{-1}] P_{-1} \bm{b}_t
\end{align*}
where $P_1 \bm{K}_i$ and $P_1 \bm{b}_t$ are in the same $1$ dimensional linear space, thus
\begin{align*}
     & P_1 \bm{K}_i \bm{K}_i^T P_1 P_1 \bm{b}_t\\
     =& \|P_1 \bm{K}_i\|_2 \|P_1 \bm{b}_t\|_2 sign(\langle P_1 \bm{K}_i,P_1 \bm{b}_t\rangle) P_1 \bm{K}_i\\
     =& \|P_1 \bm{K}_i\|_2^2 P_1 \bm{b}_t\\
     \Rightarrow \quad [I& - \eta P_1 \bm{K}_i \bm{K}_i^T P_1] P_1 \bm{b}_t = [1 - \eta\|P_1 \bm{K}_i\|_2^2] P_1 \bm{b}_t
\end{align*}
and 
\begin{align*}
    &\|[P_1 \bm{K}_i \bm{K}_i^T P_{-1}] P_{-1} \bm{b}_t\|_2\\
    \leq &\|P_1 \bm{K}_i\|_2 \| P_{-1} \bm{K}_i\|_2 \|P_{-1} \bm{b}_t\|_2.
\end{align*}
Then 
\begin{align}
    E[\|P_1 \bm{b}_{t+1}\|_2|\bm{b}_t]\leq E[|1 - \eta\|P_1 \bm{K}_i\|_2^2|] \|P_1 \bm{b}_t\|_2 + \eta E[\|P_1 \bm{K}_i\|_2 \|P_{-1}\bm{K}_i\|_2]\|P_{-1}\bm{b}_t\|_2\label{ch2:eq:lemc1_7}\\
    E[\|P_1 \bm{b}_{t+1}\|_2|\bm{b}_t]\geq E[|1 - \eta\|P_1 \bm{K}_i\|_2^2|] \|P_1 \bm{b}_t\|_2 - \eta E[\|P_1 \bm{K}_i\|_2 \|P_{-1}\bm{K}_i\|_2]\|P_{-1}\bm{b}_t\|_2\label{ch2:eq:lemc1_8}
\end{align}
where
\begin{align*}
    &E[|1 - \eta\|P_1 \bm{K}_i\|_2^2|]\\ =&\frac{1}{n} \sum_{i=1}^n |1 - \eta\|P_1 \bm{K}_i\|_2^2|\\
    =&\frac{n-1}{n} + \frac{1}{n}|1 - \eta\|P_1 \bm{K}_1\|_2^2|\\
    :=& q_1(\eta)
\end{align*}
and
\begin{align*}
    &\eta E[\|P_1 \bm{K}_i\|_2 \|P_{-1}\bm{K}_i\|_2] \\
    =& \eta \frac{1}{n}\sum_{i=1}^n \|P_1 \bm{K}_i\|_2 \| P_{-1}\bm{K}_i\|_2\\
    =&\frac{\eta}{n}\|P_1 \bm{K}_1\|_2 \|P_{-1} \bm{K}_1\|_2\\
    \stackrel{}{\leq}&\frac{\eta}{n}(\lambda_1 + \sqrt{n} \tau)c_3  (2\lambda_1 + (n-2)\tau) \sqrt{n}\tau\\
    \leq & \frac{\eta}{n}c_4  \sqrt{n} \tau := \xi(\eta)\numberthis \label{ch2:eq:lemc1_9}
\end{align*}
where the first inequality by upper bounds \eqref{ch2:eq:lemB3_eq3} and \eqref{ch2:eq:lemB3_eq4}, second inequality by $n\tau \leq \mathcal{O}(1)$. Plug the term \eqref{ch2:eq:lemc1_9} into inequalities \eqref{ch2:eq:lemc1_7} and \eqref{ch2:eq:lemc1_8}, take expectation on both sides, we get claims \eqref{ch2:eq:lemc1_1} and \eqref{ch2:eq:lemc1_2}.\\
\\
For inequality \eqref{ch2:eq:lemc1_3}, check
\begin{align*}
    P_{-1} \bm{b}_{t+1} &= P_{-1}[I-\eta \bm{K}_i \bm{K}_i^T]\bm{b}_t\\
    &= [I - \eta P_{-1} \bm{K}_i \bm{K}_i^T P_{-1}] P_{-1} \bm{b}_t - \eta [P_{-1} \bm{K}_i \bm{K}_i^T P_{1}] P_{1} \bm{b}_t
\end{align*}
Then we have
\begin{align*}
 &\quad E[\|P_{-1} \bm{b}_{t+1}\|_2|\bm{b}_t]\\
 &\leq \frac{1}{n}\sum_{i=1}^n \|[I - \eta P_{-1} \bm{K}_i \bm{K}_i^T P_{-1}] P_{-1} \bm{b}_t\|_2 + \eta E[\|P_{-1} \bm{K}_i\|_2 \|P_{1}\bm{K}_i\|_2]\|P_{1}\bm{b}_t\|_2\\
&\stackrel{\eqref{ch2:eq:lemc1_9}}{\leq} \frac{1}{n}\sum_{i=1}^n \|[I - \eta P_{-1} \bm{K}_i \bm{K}_i^T P_{-1}] P_{-1} \bm{b}_t\|_2 + \xi(\eta)\|P_{1}\bm{b}_t\|_2
\end{align*}
where
\begin{align*}
    &\frac{1}{n}\sum_{i=1}^n \|[I - \eta P_{-1} \bm{K}_i \bm{K}_i^T P_{-1}] P_{-1} \bm{b}_t\|_2\\
    =&\frac{1}{n}\sum_{i=1}^n \sqrt{\|[I - \eta P_{-1} \bm{K}_i \bm{K}_i^T P_{-1}] P_{-1} \bm{b}_t\|_2^2}\\
    =&\frac{1}{n}\sum_{i=1}^n \sqrt{\|P_{-1} \bm{b}_t\|_2^2 + \eta^2\|P_{-1} \bm{K}_i \bm{K}_i^T P_{-1} P_{-1} \bm{b}_t\|_2^2 - 2\eta \langle P_{-1} \bm{b}_t, P_{-1} \bm{K}_i \bm{K}_i^T P_{-1} P_{-1} \bm{b}_t\rangle}\\
    =&\frac{1}{n}\sum_{i=1}^n \sqrt{\|P_{-1} \bm{b}_t\|_2^2 + \eta^2( \bm{K}_i^T P_{-1} P_{-1} \bm{b}_t)^2\|P_{-1} \bm{K}_i\|_2^2 - 2\eta( \bm{K}_i^T P_{-1} P_{-1} \bm{b}_t)^2}\\
    =& \frac{1}{n}\sum_{i=1}^n\sqrt{1 + \frac{( \bm{K}_i^T P_{-1} P_{-1} \bm{b}_t)^2}{\|P_{-1} \bm{b}_t\|_2^2} (\eta^2\|P_{-1} \bm{K}_i\|_2^2 - 2\eta)} \|P_{-1} \bm{b}_t\|_2\\
    \leq & \sqrt{1 + \frac{1}{n}\sum_{i=1}^n\frac{( \bm{K}_i^T P_{-1} P_{-1} \bm{b}_t)^2}{\|P_{-1} \bm{b}_t\|_2^2} (\eta^2\|P_{-1} \bm{K}_i\|_2^2 - 2\eta)} \|P_{-1} \bm{b}_t\|_2
\end{align*}
where the last inequality from Jensen's inequality. Now the term
\begin{align*}
    &\frac{1}{n}\sum_{i=1}^n\frac{( \bm{K}_i^T P_{-1} P_{-1} \bm{b}_t)^2}{\|P_{-1} \bm{b}_t\|_2^2} (\eta^2\|P_{-1} \bm{K}_i\|_2^2 - 2\eta)\\
    \stackrel{\eqref{ch2:eq:lemC1eq1}}{\leq}&\frac{1}{n}\sum_{i=1}^n\frac{( \bm{K}_i^T P_{-1} P_{-1} \bm{b}_t)^2}{\|P_{-1} \bm{b}_t\|_2^2} \left[\eta^2(\lambda_2^2 + (n-1)\tau^2) - 2\eta\right]\\
    = &\frac{\| K P_{-1} \bm{b}_t\|_2^2}{n\|P_{-1} \bm{b}_t\|_2^2} \left[\eta^2(\lambda_2^2 + (n-1)\tau^2) - 2\eta\right]
\end{align*}
Let 
$$q_{-1}(\eta) := \sqrt{1 + \frac{\| K P_{-1} \bm{b}_t\|_2^2}{n\|P_{-1} \bm{b}_t\|_2^2} \left[\eta^2(\lambda_2^2 + (n-1)\tau^2) - 2\eta\right]}$$
combine all three inequalities above, take the expectation w.r.t. $\bm{b}_t$, we have claim \eqref{ch2:eq:lemc1_3}. 

\end{proof}

\begin{lemma}[Long run behavior of SGD with moderate step size]\label{ch2:lem:E2}
Assume $\bm{b}_0$ is away from $0$, $\lambda_n^2 >  (2\lambda_1 + (n-2)\tau)n\tau + c_4 \sqrt{n}\tau$, $\lambda_2^2 + c_6\sqrt{n}\tau < \lambda_1^2 - c_5 \sqrt{n}\tau$ where $c_5, c_6$ are constants such that
$$c_5 \geq  c_3^2 [2\lambda_1 + (n-2)\tau]^2 \sqrt{n}\tau + c_4$$
$$c_6 \geq\frac{\sqrt{n}\tau -c_4^2 n^{-.5} \tau/[\lambda_n^2 - (2\lambda_1 + (n-2)\tau)n\tau]+ \lambda_2^2c_4/[\lambda_n^2 - [2\lambda_1 + (n-2)\tau]n\tau]}{1- c_4 \sqrt{n}\tau/[\lambda_n^2 - (2\lambda_1 + (n-2)\tau)n\tau]}$$
Consider first $k_1$ steps of SGD updates with step size $\eta$:
$$\frac{2}{\lambda_1^2 - c_5\sqrt{n}\tau} < \eta <\frac{2}{\lambda_2^2 + c_6 \sqrt{n}\tau}$$
Fix a $\beta_0\leq A_0$, then for $0<\epsilon<1$ and $0<\beta<\beta_0$ such that $\sqrt{n}\tau \leq poly(\epsilon\beta)$, there exists $k_1 = \mathcal{O}(\log\frac{1}{\epsilon\beta})$ satisfying:
\begin{itemize}
    \item $B_{k_1} \leq \epsilon\beta$
    \item $A_{k_1} \leq \|\bm{b}_0\|_2 *  \rho_1^{k_1} + \epsilon\beta/2$ for some $\rho_1>1$
    \item $A_k > \beta_0$ for $k=0,\ldots,k_1$.
\end{itemize}
\end{lemma}
\begin{proof}
For this choice of $\eta$, denote $q_1 = q_1(\eta), q_{-1} = q_{-1}(\eta), \xi = \xi(\eta)$. By Lemma \ref{ch2:lem:C1}, we have
\begin{align*}
    A_{k} &\geq q_1 A_{k-1} - \xi B_{k-1}\\
   \begin{bmatrix}A_{k}\\B_{k}\end{bmatrix} &\leq \begin{bmatrix}   q_1 & \xi\\ \xi& q_{-1}   \end{bmatrix}\begin{bmatrix}A_{k-1}\\B_{k-1}\end{bmatrix}
\end{align*}
Decompose the coefficient matrix as
\begin{align*}
    \begin{bmatrix}   q_1 & \xi\\ \xi& q_{-1}   \end{bmatrix} = \begin{bmatrix}   \cos \theta & -\sin \theta\\ \sin \theta& \cos \theta  \end{bmatrix}\begin{bmatrix}\rho_{1}& 0 \\0& \rho_{-1}\end{bmatrix} \begin{bmatrix}   \cos \theta & \sin \theta\\ -\sin \theta& \cos \theta  \end{bmatrix} 
\end{align*}
Assume w.l.o.g. that $\sin \theta \geq 0$ ( since otherwise we can take $\theta \to \theta + \pi$), then we have
\begin{align*}
       \begin{bmatrix}A_{k}\\B_{k}\end{bmatrix} &\leq \begin{bmatrix}   q_1 & \xi\\ \xi& q_{-1}   \end{bmatrix}^k\begin{bmatrix}A_{0}\\B_{0}\end{bmatrix}\\
       &=\begin{bmatrix}   \cos \theta & -\sin \theta\\ \sin \theta& \cos \theta  \end{bmatrix}\begin{bmatrix}\rho_{1}^k& 0 \\0& \rho_{-1}^k\end{bmatrix} \begin{bmatrix}   \cos \theta & \sin \theta\\ -\sin \theta& \cos \theta  \end{bmatrix} \begin{bmatrix}A_{0}\\B_{0}\end{bmatrix}\\
       &= \begin{bmatrix}A_{0}(\rho_1^k\cos^2\theta + \rho_{-1}^k\sin^2 \theta) + B_{0}(\rho_1^k \cos\theta \sin\theta - \rho_{-1}^k \cos\theta \sin\theta)\\B_{0}(\rho_{-1}^k\cos^2\theta + \rho_{1}^k\sin^2 \theta) + A_0(\rho_1^k \cos\theta \sin\theta - \rho_{-1}^k\cos\theta \sin\theta)\end{bmatrix}\\
       &= \begin{bmatrix}A_{0}\rho_1^k +       (\rho_1^k - \rho_{-1}^k)\sin\theta (B_0\cos\theta - A_0\sin\theta) \\
       B_{0}\rho_{-1}^k + (\rho_1^k - \rho_{-1}^k)\sin\theta (B_0\sin\theta + A_0\cos\theta)\end{bmatrix}\\
       &\leq \begin{bmatrix}A_{0}\rho_1^k +       |\rho_1^k - \rho_{-1}^k|\sin\theta \sqrt{B_0^2 + A_0^2} \\
       B_{0}\rho_{-1}^k + |\rho_1^k - \rho_{-1}^k|\sin\theta\sqrt{B_0^2 + A_0^2}\end{bmatrix}\\
        &= \begin{bmatrix}A_{0}\rho_1^k +       |\rho_1^k - \rho_{-1}^k|\sin\theta\|\bm{b}_0\|_2\\
       B_{0}\rho_{-1}^k + |\rho_1^k - \rho_{-1}^k|\sin\theta\|\bm{b}_0\|_2\end{bmatrix}
\end{align*}
We claim the following holds:
\begin{subequations}\label{ch2:eq:lemc2_eq5}
\begin{align}
    &0<\rho_{-1} < 1 < \rho_{1}\leq q_1 + \xi\label{ch2:eq:lemc2_eq5_a}\\
    &\rho_{-1}^{k_1}\|\bm{b}_0\|_2\leq \epsilon\beta/2\label{ch2:eq:lemc2_eq5_b}\\
    &\rho_{1}^{k_1}\|\bm{b}_0\|_2\sin\theta\leq \epsilon\beta/2\label{ch2:eq:lemc2_eq5_c}\\
    &(B_0 + \epsilon\beta_0/2)\xi < (q_1 - 1)\beta_0\label{ch2:eq:lemc2_eq5_d}
\end{align}
\end{subequations}
which we check later. Using inequalities \eqref{ch2:eq:lemc2_eq5}, we can upper bound $B_{k_1}$ as
\begin{align*}
B_{k_1} &\leq  B_{0}\rho_{-1}^{k_1} + (\rho_1^{k1} - \rho_{-1}^{k_1})\sin\theta\|\bm{b}_0\|_2\\
&\leq \|\bm{b}_0\|_2\rho_{-1}^{k_1} + \rho_1^{k1}\sin\theta\|\bm{b}_0\|_2\\
&\stackrel{\eqref{ch2:eq:lemc2_eq5_b},\eqref{ch2:eq:lemc2_eq5_c}}{\leq } \epsilon\beta
\end{align*}

In addition, for $k=0,\ldots,k_1$
\begin{align*}
B_{k} &\leq  B_{0}\rho_{-1}^{k} + (\rho_1^{k} - \rho_{-1}^{k})\sin\theta\|\bm{b}_0\|_2\\
&\leq B_{0} + \rho_1^{k1}\sin\theta\|\bm{b}_0\|_2\\
&\stackrel{\eqref{ch2:eq:lemc2_eq5_c}}{\leq } B_0 + \epsilon\beta/2
\end{align*}
We now lower bound $A_k$ by mathematical induction. We have $A_0 \geq \beta_0$, assume $A_{k-1} > \beta_0$, then 
\begin{align*}
    A_k &\geq q_1 A_{k-1} - \xi B_{k-1}\\
    &\geq q_1 \beta_0 - \xi (B_0 + \epsilon\beta/2)\\
    &\stackrel{\eqref{ch2:eq:lemc2_eq5_d}}{>} q_1 \beta_0 - (q_1 - 1)\beta_0 = \beta_0
\end{align*}
For upper bound $A_{k_1}$, check
\begin{align*}
    A_{k_1} &\leq A_{0}\rho_1^{k_1} +       (\rho_1^{k_1} - \rho_{-1}^{k_1})\sin\theta\|\bm{b}_0\|_2\\
    &\leq \|\bm{b}_0\|_2\rho_1^{k_1} + \rho_1^{k_1} \sin\theta\|\bm{b}_0\|_2\\
    &\stackrel{\eqref{ch2:eq:lemc2_eq5_c}}{\leq} \|\bm{b}_0\|_2\rho_1^{k_1} +\epsilon\beta/2
\end{align*}
We have all lemma claims proved. Now it remains to check inequalities \eqref{ch2:eq:lemc2_eq5}. First note that our choice of the upper bound on $\eta$ guarantees that $q_{-1}< 1$. \\
\\
\textbf{For inequality \eqref{ch2:eq:lemc2_eq5_a}:} By Gershgorin circle theorem, it suffices to show 
$q_{-1} + \xi < 1, q_1 - \xi > 1$, then we have $\rho_1 \geq q_1 - \xi > 1$ and $\rho_{-1} \leq q_{-1} + \xi < 1$. In addition, we need the matrix to be positive definite so that $\rho_{-1} >0$, we just need $q_1 q_{-1} > \xi^2$ to make the matrix p.d..
\begin{align*}
    &q_{-1} + \xi < 1\\
    \Longleftrightarrow&\sqrt{1 + \frac{\| K P_{-1} \bm{b}_t\|_2^2}{n\|P_{-1} \bm{b}_t\|_2^2} \left[\eta^2(\lambda_2^2 + (n-1)\tau^2) - 2\eta\right]} < 1 - c_4 \eta n^{-1/2}\tau\\
    \Longleftarrow&\frac{\| K P_{-1} \bm{b}_t\|_2^2}{n\|P_{-1} \bm{b}_t\|_2^2} \left[\eta^2(\lambda_2^2 + (n-1)\tau^2) - 2\eta\right] < c_4^2 \eta^2 n^{-1}\tau^2 - 2 c_4 \eta n^{-1/2}\tau\\
    \stackrel{\text{Lemma } \ref{ch2:lem:C4}}{\Longleftarrow} &2c_4\eta n^{-1/2}\tau \leq c_4^2\eta^2 n^{-1} \tau^2 + \frac{\lambda_n^2 - [2\lambda_1 + (n-2)\tau]n\tau}{n}\left[2\eta - \eta^2(\lambda_2^2 + (n-1)\tau^2)\right]\\
    \Longleftrightarrow & \frac{\lambda_n^2 - [2\lambda_1 + (n-2)\tau]n\tau}{n}\left[\lambda_2^2 + (n-1)\tau^2 \right] \eta -c_4^2 n^{-1} \tau^2\eta \\
    &\quad\leq 2\frac{\lambda_n^2 - [2\lambda_1 + (n-2)\tau]n\tau}{n} - 2c_4 n^{-1/2}\tau\\
    \Longleftarrow& \eta \leq 2\frac{1- c_4 \sqrt{n}\tau/[\lambda_n^2 - [2\lambda_1 + (n-2)\tau]n\tau]}{\lambda_2^2 + (n-1)\tau^2 -c_4^2 \tau^2/[\lambda_n^2 - [2\lambda_1 + (n-2)\tau]n\tau]}\\
    \Longleftrightarrow& \eta \leq \frac{ 2}{\lambda_2^2  + \frac{(n-1)\tau^2 -c_4^2 \tau^2/[\lambda_n^2 - [2\lambda_1 + (n-2)\tau]n\tau]+ \lambda_2^2 [c_4 \sqrt{n}\tau/[\lambda_n^2 - [2\lambda_1 + (n-2)\tau]n\tau]]}{1- c_4 \sqrt{n}\tau/[\lambda_n^2 - [2\lambda_1 + (n-2)\tau]n\tau]}}
\end{align*}
which is true by our choice of $\eta$.
\begin{align*}
    &q_{1} - \xi > 1\\
    \Longleftrightarrow & \frac{n-1}{n} + \frac{1}{n}|1 - \eta\|P_1 \bm{K}_1\|_2^2| - c_4 \eta n^{-1/2}\tau >  1\\
    \Longleftrightarrow&1 + c_4 \eta \sqrt{n}\tau < |1 - \eta\|P_1 \bm{K}_1\|_2^2|\\
    \Longleftarrow& \eta\|P_1 \bm{K}_1\|_2^2 - 1 > 1 + c_4 \eta \sqrt{n}\tau\\
    \Longleftarrow&\eta > \frac{2}{\|P_1 \bm{K}_1\|_2^2 - c_4 \sqrt{n}\tau}\\
    \stackrel{\eqref{ch2:eq:lemB3_eq4}}{\Longleftarrow}&\eta > \frac{2}{\lambda_1^2 - c_3^2 [2\lambda_1 + (n-2)\tau]^2 n\tau^2 - c_4 \sqrt{n}\tau}
\end{align*}
which is true by our lower bound on $\eta$.
\begin{align*}
    &q_1 q_{-1} > \xi^2\\
    \Longleftarrow & q_{-1}^2 \geq \xi\\
    \Longleftrightarrow& 1 + \frac{\| K P_{-1} \bm{b}_t\|_2^2}{n\|P_{-1} \bm{b}_t\|_2^2} \left[\eta^2(\lambda_2^2 + (n-1)\tau^2) - 2\eta\right] > c_4 \eta n^{-1/2} \tau\\
    \Longleftarrow & 1 -2 \eta \frac{\lambda_2^2 + (2\lambda_1 +(n-1)\tau) n\tau }{n}  > c_4 \eta n^{-1/2} \tau\\
    \Longleftrightarrow& \eta < \frac{n}{2\lambda_2^2 + 2(2\lambda_1 +(n-1)\tau) n\tau + c_4 \sqrt{n} \tau}
\end{align*}
which is true.\\
\\
\textbf{For inequality \eqref{ch2:eq:lemc2_eq5_b}:} It suffices to take 
\begin{align*}
    k_1 = \frac{\log(\epsilon \beta/(2\|\bm{b}_0\|_2))}{\log(\rho_{-1})} = \mathcal{O}(\log\frac{1}{\epsilon\beta})
\end{align*}
\\
\\
\textbf{For inequality \eqref{ch2:eq:lemc2_eq5_c}:} We just need to show $\sin\theta < (\rho_{-1}/\rho_1)^{k_1}$. Calculate that $\xi / (q_1 - q_{-1}) = \frac{\cos\theta \sin\theta}{\cos^2\theta - \sin^2\theta}$, then 
\begin{align*}
    &\sin\theta < (\rho_{-1}/\rho_1)^{k_1}\\
    \Longleftarrow&\xi / (q_1 - q_{-1}) < 0.9 (\rho_{-1}/\rho_1)^{k_1}\\
    \Longleftarrow& \xi < 0.9(q_1 - q_{-1})(\frac{q_{-1} - \xi}{q_1 + \xi})^{k_1}\\
    \Longleftarrow& \xi < 0.9(q_1 - q_{-1})(\epsilon \beta/(2\|\bm{b}_0\|_2))^{ 1 - \frac{\log(q_1 + \xi)}{\log(q_{-1}-\xi)}}\\
    \Longleftarrow & \xi < (q_1 - 1)poly(\epsilon\beta)\\
    \Longleftarrow & \sqrt{n}\tau \leq poly(\epsilon\beta)
\end{align*}
\\
\\
\textbf{For inequality \eqref{ch2:eq:lemc2_eq5_d}:} Suffice to show 
\begin{align*}
    &\xi < \frac{(q_1 - 1)\beta_0}{B_0 + \epsilon\beta_0/2}\\
    \Longleftarrow &\sqrt{n}\tau \leq\mathcal{O}(1)
\end{align*}
\end{proof}

\begin{lemma}[Long run behavior of SGD with small step size]\label{ch2:lem:E3}
Under the same notations and assumptions as Lemma \ref{ch2:lem:E2} unless otherwise specified. Consider another $k_2 - k_1$ steps of SGD update with step size 
$$\eta' < \frac{1}{\lambda_1^2 + c_7 \sqrt{n}\tau}$$
where the constant $c_7\geq \sqrt{n}\tau  +c_4$.
Then we have for $k > k_1$:
\begin{itemize}
    \item $B_{k} \leq \epsilon\beta$
    \item $A_k \leq \left\{\begin{array}{ll}
        q A_{k-1} &, A_{k-1}>\beta  \\
        \beta &, A_{k-1}<\beta
    \end{array}\right.$ 
\end{itemize} where $q:= q_1(\eta') + \xi(\eta')\epsilon<1$.
\end{lemma}
\begin{proof}
Denote $q_1' = q_1(\eta'), q_{-1}' = q_{-1}(\eta'), \xi' = \xi(\eta')$, then $q =q_1' + \xi'\epsilon $, denote $B = \|\bm{b}_0\|_2 *  \rho_1^{k_1} + \epsilon\beta/2$. We have by proof of Lemma \ref{ch2:lem:E2} that $q_{-1}' + \xi < 1$. We claim the following holds:
\begin{subequations}\label{ch2:eq:lemc3_eq1}
\begin{align}
&q < 1\label{ch2:eq:lemc3_eq1_a}\\
&\xi'B \leq (1 - q_1')\epsilon\beta\label{ch2:eq:lemc3_eq1_b}
\end{align}
\end{subequations}
\textbf{Check inequality} \eqref{ch2:eq:lemc3_eq1_a}:
\begin{align*}
    &q_1' + \xi'\epsilon < 1\\
    \Longleftarrow & \frac{n-1}{n} + \frac{1}{n}|1 - \eta'\|P_1 \bm{K}_1\|_2^2| + c_4 \eta'  n^{-1/2} \tau < 1\\
    \Longleftrightarrow &|1 - \eta'\|P_1 \bm{K}_1\|_2^2| < 1 - c_4 \eta'  \sqrt{n} \tau\\
    \Longleftrightarrow &c_4 \eta'  \sqrt{n} \tau  <  \eta'\|P_1 \bm{K}_1\|_2^2 < 2 - c_4 \eta'  \sqrt{n} \tau\\
    \Longleftrightarrow &\left\{\begin{array}{l}
    c_4  \sqrt{n} \tau  <  \|P_1 \bm{K}_1\|_2^2\\
     \eta' < \frac{2}{\|P_1 \bm{K}_1\|_2^2 +c_4 \sqrt{n} \tau }
    \end{array}\right.\\
    \Longleftarrow &\left\{\begin{array}{l}
    \sqrt{n} \tau  < \mathcal{O}(1)\\
     \eta' < \frac{2}{\lambda_1^2 + n\tau^2 +c_4 \sqrt{n} \tau }
    \end{array}\right.
\end{align*}
which are true by assumption.\\
\\
\textbf{Check inequality} \eqref{ch2:eq:lemc3_eq1_b}:
\begin{align*}
    &\xi'B \leq (1 - q_1')\epsilon \beta\\
    \Longleftrightarrow& c_4 \eta'  n^{-1/2} \tau (\|\bm{b}_0\|_2 *  \rho_1^{k_1} + \epsilon\beta/2)\leq (1 - \frac{n-1}{n} - \frac{1}{n}(1 - \eta'\|P_1\bm{K}_1\|_2^2)) \epsilon\beta\\
    \Longleftrightarrow& c_4 \eta'  \sqrt{n} \tau (\|\bm{b}_0\|_2 *  \exp(k_1)^{\log\rho_1} + \epsilon\beta/2)\leq  \eta'\|P_1\bm{K}_1\|_2^2 \epsilon\beta\\
    \Longleftarrow & \sqrt{n}\tau \leq poly(\epsilon\beta).
\end{align*}
With \eqref{ch2:eq:lemc3_eq1} we can prove the lemma by mathematical induction. Suppose $B_{k-1} \leq \epsilon\beta$ and $A_{k-1} \leq A_{k_1} \leq B$, then check
\begin{align*}
    B_{k} &\leq \xi' A_{k-1} + q_{-1}' B_{k-1}\\
    &\leq \xi' B + q_{-1}' \epsilon\beta\\
    &\stackrel{\eqref{ch2:eq:lemc3_eq1_b}}{\leq}\epsilon\beta
\end{align*}
and 
\begin{align*}
    A_{k} &\leq q_{1}' A_{k-1} + \xi' B_{k-1}\\
    &\leq q_{1}' A_{k-1} + \xi' \epsilon\beta\\
    &\leq (q_{1}' + \xi' \epsilon)\max\{A_{k-1},\beta\}\\
    &\leq \left\{\begin{array}{ll}
        q A_{k-1} &, A_{k-1}>\beta  \\
        q\beta < \beta &, A_{k-1}<\beta
    \end{array}\right..
\end{align*}
\end{proof}

We recap Theorem \ref{ch2:thm:informal1} using our notations in previous lemmas as follows: 
\begin{theorem}[Directional bias of the two-stage SGD]\label{ch2:thm:E4}
Use the two stage SGD scheme as defined in Lemma \ref{ch2:lem:E2} and \ref{ch2:lem:E3}. Assume $n\tau < poly(\epsilon)$, then there exists $k_1 = \mathcal{O}(\log\frac{1}{\epsilon})$ and $k_2$ such that
$$(1-2\epsilon)\gamma_1\leq \frac{E[\|K \bm{b}_{k_2}\|_2]}{E[\|\bm{b}_{k_2}\|_2]} \leq \gamma_1$$ 
where $\gamma_1$ is the largest eigenvalue of $K$.
\end{theorem}
\begin{proof}
In Lemma \ref{ch2:lem:E2} let $\beta = \beta_0$, then for $k_1 = \mathcal{O}(\log\frac{1}{\epsilon})$ we have $B_{k_1} \leq \epsilon\beta_0$. For the 2nd stage, by Lemma \ref{ch2:lem:E3} we can early stop at $k_2$ such that $A_{k_2} \geq \beta_0$ and $A_{k_2 + 1} < \beta_0$. We then have
\begin{align*}
    B_{k_2}\leq  \epsilon\beta_0 \leq \epsilon A_{k_2}
\end{align*}
Then we check
\begin{align*}
    &\frac{E\|K \bm{b}_{k_2}\|_2}{E\|\bm{b}_{k_2}\|_2}\\
    =& \frac{E\sqrt{\|K P_{-1} \bm{b}_{k_2}\|_2^2 + \|K P_{1} \bm{b}_{k_2}\|_2^2 + 2\langle \bm{K}_1^T P_{-1} \bm{b}_{k_2},\bm{K}_1^T P_{1} \bm{b}_{k_2}\rangle}}{E\|\bm{b}_{k_2}\|_2}\\
    \geq & \frac{E\sqrt{\|K P_{1} \bm{b}_{k_2}\|_2^2 -2 \|P_{-1}\bm{K}_1\|_2 \|P_{1}\bm{K}_1\|_2 \| \bm{b}_{k_2}\|_2^2}}{E\|\bm{b}_{k_2}\|_2}\\
    \geq & \frac{E\sqrt{\|K P_{1} \bm{b}_{k_2}\|_2^2} -E\sqrt{2\|P_{-1}\bm{K}_1\|_2 \|P_{1}\bm{K}_1\|_2 \| \bm{b}_{k_2}\|_2^2}}{E\|\bm{b}_{k_2}\|_2}\\
    = & \frac{E\|\bm{K}_1^T P_{1} \bm{b}_{k_2}\|_2 -\sqrt{2\|P_{-1}\bm{K}_1\|_2 \|P_{1}\bm{K}_1\|_2}E \| \bm{b}_{k_2}\|_2}{E\|\bm{b}_{k_2}\|_2}\\
    \stackrel{\eqref{ch2:eq:lemB3_eq3},\eqref{ch2:eq:lemB3_eq4}}{\geq} & \sqrt{\lambda_1^2 - c_3^2 [2\lambda_1 + (n-2)\tau]^2 n\tau^2 }\frac{E\| P_{1} \bm{b}_{k_2}\|_2}{E\|P_1 \bm{b}_{k_2}\|_2 + E\|P_{-1} \bm{b}_{k_2}\|_2}\\
    &\quad - \sqrt{2 (\lambda_1 + \sqrt{n} \tau)(c_3  (2\lambda_1 + (n-2)\tau) \sqrt{n}\tau)}\\
    \geq &\sqrt{\lambda_1^2 - c_3^2 [2\lambda_1 + (n-2)\tau]^2 n\tau^2 }\frac{\beta_0}{\epsilon\beta_0 + \beta_0} - \sqrt{2 (\lambda_1 + \sqrt{n} \tau)(c_3  (2\lambda_1 + (n-2)\tau) \sqrt{n}\tau)}\\
    \stackrel{\eqref{ch2:eq:lemB2_eq1}}{\geq} &(\gamma_1 -  n\tau - c_3(2\lambda_1 + (n-2)\tau) \sqrt{n}\tau )(1 - \epsilon) -\sqrt{2 (\lambda_1 + \sqrt{n} \tau)(c_3  (2\lambda_1 + (n-2)\tau) \sqrt{n}\tau)}\\
    \geq& \gamma_1 (1 - \epsilon) -\gamma_1\epsilon\quad (\text{By } n\tau < poly(\epsilon))\\
    =& \gamma_1 (1 - 2\epsilon)
\end{align*}
And the upper bound in the theorem is by definition of $\gamma_1$. 
\end{proof}

\section{Directional bias of GD with moderate or small step size}\label{ch2:app:F}
This section includes the proof of Theorem \ref{ch2:thm:informal2}. We first rewrite the GD updates as linear combination of eigenvectors. Then the theorem is proved using the transformed variables and finally transformed back to original parameters.

\textbf{The directional bias of GD does not require diagonal dominant Gram matrix.}

\textbf{Reloading notations}
Denote the eigen decomposition of $K$:
\begin{align*}
    K = G \Gamma G^T, \Gamma = diag(\gamma_1, \ldots,\gamma_n), G = [\bm{g}_1,\ldots,\bm{g}_n]
\end{align*}
where the eigenvectors $\bm{g}_i$'s are orthogonal. The GD update as
\begin{align*}
    \bm{\alpha}_{t+1} = \bm{\alpha}_t - \frac{\eta}{n} K (K \bm{\alpha}_t - \bm{y})
\end{align*}
Denote $\bm{w}_t := G^T(\bm{\alpha}_t - \hat{\bm{\alpha}})$, we can rewrite GD update in $\bm{w}_t$:
\begin{align*}
    \bm{w}_{t + 1} = \bm{w}_t - \frac{\eta}{n}\Gamma^2 \bm{w}_t = ( I - \frac{\eta}{n}\Gamma^2) \bm{w}_t
\end{align*}
We recap Theorem \ref{ch2:thm:informal2} to make reading easier as follows:
\begin{theorem}[Direction bias of GD]\label{ch2:thm:D1}
Assume $\bm{\alpha}_0$ is away from $0$, $ \lambda_n + 2 n \tau  < \lambda_{n-1}$, GD with step size:
\begin{align*}
   \eta< \frac{n}{(\lambda_1 + n\tau)^2}
\end{align*}
For a small $\epsilon > 0$, take $k = \mathcal{O}(\log \frac{1}{\epsilon})$, we have
\begin{align*}
    \gamma_n\leq \frac{\|K (\bm{\alpha}_k - \hat{\bm{\alpha}})\|_2}{\|\bm{\alpha}_k - \hat{\bm{\alpha}}\|_2}\leq \sqrt{1 + \epsilon}\gamma_n
\end{align*}
\end{theorem}
\begin{proof}
For $i = 1,\ldots,n$, we have
\begin{align*}
    \begin{split}
        w_{k}^{(i)} = (1 - \eta\gamma_i^2/n)^{k} w_{0}^{(i)}
    \end{split}
\end{align*}
Denote $q_i = 1 - \eta\gamma_i^2/n$, then $0< q_1 \leq \ldots \leq q_n < 1$ since 
\begin{align*}
    0<\eta< \frac{n}{(\lambda_1 +  n\tau)^2} \stackrel{\eqref{ch2:eq:lemB2_eq1}}{\leq} \frac{n}{\gamma_1^2} \leq \frac{n}{\gamma_i^2}
\end{align*}
Since $ \lambda_n +  n \tau  < \lambda_{n-1} - n\tau$, we have $\gamma_n < \gamma_{n-1}$ by lemma \ref{ch2:lem:B2}, it follows that $q_n > q_{n-1}$. Denote $q = q_{n-1}/q_n < 1$, then 
\begin{align*}
    &\frac{\sum_{i = 1}^{n-1} (w_k^{(i)})^2}{(w_k^{(n)})^2}\\
    =&\frac{\sum_{i = 1}^{n-1}q_i^{2k} (w_0^{(i)})^2}{q_n^{2k}(w_0^{(n)})^2}\\
    \leq &\frac{\sum_{i = 1}^{n-1}q_{n-1}^{2k} (w_0^{(i)})^2}{q_n^{2k}(w_0^{(n)})^2}\\
    =&q^{2k}\frac{\sum_{i = 1}^{n-1} (w_0^{(i)})^2}{(w_0^{(n)})^2}
\end{align*}
Let $q^{2k} \leq \frac{\gamma_n^2 \epsilon(w_0^{(n)})^2}{\gamma_1^2 \sum_{i = 1}^{n-1} (w_0^{(i)})^2} \Longleftrightarrow k \geq \frac{1}{2} \frac{\log \frac{\gamma_n^2 \epsilon(w_0^{(n)})^2}{\gamma_1^2 \sum_{i = 1}^{n-1} (w_0^{(i)})^2}}{\log q} = \mathcal{O}(\log \frac{1}{\epsilon})$, we have
\begin{align*}
    \frac{\sum_{i = 1}^{n-1} (w_k^{(i)})^2}{(w_k^{(n)})^2}\leq \frac{\gamma_n^2\epsilon}{\gamma_1^2}
\end{align*}
Thus 
\begin{align*}
\frac{\|K (\bm{\alpha}_k - \hat{\bm{\alpha}})\|_2^2}{\|\bm{\alpha}_k - \hat{\bm{\alpha}}\|_2^2} &= \frac{\|\Gamma \bm{w}_k\|_2^2}{\|\bm{w}_k\|_2^2}\\
&=\frac{\sum_{i = 1}^{n} (w_k^{(i)})^2\gamma_i^2}{\sum_{i = 1}^{n} (w_k^{(i)})^2}\\
&=\frac{ (w_k^{(n)})^2\gamma_n^2}{\sum_{i = 1}^{n} (w_k^{(i)})^2} + \frac{\sum_{i = 1}^{n-1} (w_k^{(i)})^2\gamma_i^2}{\sum_{i = 1}^{n} (w_k^{(i)})^2}\\
&\leq\gamma_n^2 + \frac{\sum_{i = 1}^{n-1} (w_k^{(i)})^2}{\sum_{i = 1}^{n} (w_k^{(i)})^2}\gamma_1^2\\
&\leq \gamma_n^2 + \gamma_1^2\frac{\gamma_n^2}{\gamma_1^2}\epsilon = \gamma_n^2(1+\epsilon)
\end{align*}
thus $$\frac{\|K (\bm{\alpha}_k - \hat{\bm{\alpha}})\|_2}{\|\bm{\alpha}_k - \hat{\bm{\alpha}}\|_2}\leq\sqrt{\gamma_n^2(1+\epsilon)}$$
The lower bound of the theorem holds by definition of $\gamma_n$.
\end{proof}

\section{Effect of directional bias}\label{ch2:app:G}
In this section, we provide the proof for theorems in Section \ref{ch2:sec:3.2}. There are two theorems there, so we split this section into two subsections. Subsection \ref{ch2:app:G1} proves Theorem \ref{ch2:thm:quad_loss}. For a general problem setting of squared error minimization, the Theorem gives a straightforward understanding for why directional bias towards the largest eigenvalue of the Hessian is good for generalization. Section \ref{ch2:app:G2} proves Theorem \ref{ch2:thm:generalization} by giving concrete generalization bounds of SGD and GD estimators in  kernel regression.

\subsection{Proof of Theorem \ref{ch2:thm:quad_loss}}\label{ch2:app:G1}
Denote $\bm{v} = \bm{w} - \bm{w}^*$, rewrite the objective function as
\begin{align*}
    \min_{\bm{v}} &\quad \|\bm{v}\|_2^2\\
    s.t. & \quad \|A \bm{v}\|_2^2  = a
\end{align*}
 Denote the eigen decomposition of $A^T A = Q\Gamma Q^{T}$ where $Q = [\bm{q}_1,\ldots,\bm{q}_n]$, $Q Q^T = Q^TQ = I$ and $\Gamma = diag([\rho_1,\ldots,\rho_n])$, $\rho_1\geq\ldots\geq\rho_n \geq 0$. 
Then 
$$\|A\bm{v} \|_2^2 = \sum_{i=1}^n \rho_i (\bm{q_i}^T \bm{v})^2$$
So 
$$\|A\bm{v} \|_2^2\leq \rho_1[\sum_{i=1}^n (\bm{q_i}^T \bm{v})^2] = \rho_1\bm{v}^TQ Q^T \bm{v} = \rho_1\|\bm{v}\|_2^2$$
The equality is achieved when $\bm{v}$ is in the direction of $\bm{q}_1$, and $\rho_1 = \|A^T A\|_2$. Take $L(\bm{w}) = \|A\bm{v}\|_2^2 = a$ then the theorem holds. 

\subsection{Proof of Theorem \ref{ch2:thm:generalization}}\label{ch2:app:G2}
\textbf{Calculate $\Delta_a^*$:}
Denote $f^* = \hat{\bm{\alpha}}^T K(\cdot,X) + \Tilde{f}$, then we have for a $f\in\bH_s$, $f = \bm{\alpha}^T K(\cdot,X)$, let $\bm{b} = \hat{\bm{\alpha}} - \bm{\alpha}$, then
\begin{align*}
    &\|f^* - f\|^2_{\bH}\\
    =&\|\bm{b}^T K(\cdot,X) + \Tilde{f}\|_{\bH}^2\\
    =&\|\bm{b}^T K(\cdot,X)\|_{\bH}^2 + \|\Tilde{f}\|_{\bH}^2 + 2\langle \bm{b}^T K(\cdot,X),\Tilde{f} \rangle_{\bH}
\end{align*}
where we can check
\begin{align*}
    \langle \bm{b}^T K(\cdot,X),\Tilde{f} \rangle_{\bH} &= \sum_{i=1}b_i \langle K(\cdot,\bm{x}_i),f^* - \hat{\bm{\alpha}}^T K(\cdot,X) \rangle_{\bH}\\
    &= \sum_{i=1}b_i [\langle K(\cdot,\bm{x}_i),f^*\rangle_{\bH} - \langle K(\cdot,\bm{x}_i),\hat{\bm{\alpha}}^T K(\cdot,X) \rangle_{\bH}]\\
   &=\sum_{i=1}b_i [f^*(\bm{x}_i) - \hat{\bm{\alpha}}^T K(\bm{x}_i,X)] (\text{By reproducing property}) \\
    &= \sum_{i=1}b_i [y_i - y_i] = 0
\end{align*}
And we further calculate that
\begin{align*}
    \|\bm{b}^T K(\cdot,X)\|_{\bH}^2 &=\langle \sum_{i=1}^n b_i K(\cdot,\bm{x}_i) ,\sum_{j=1}^n b_j K(\cdot,\bm{x}_j) \rangle_{\bH}\\
    &=\sum_{i,j = 1}^n b_i b_j \langle K(\cdot,\bm{x}_i), K(\cdot,\bm{x}_j)\rangle\\
    &= \sum_{i,j = 1}^n b_i b_j K(\bm{x}_i,\bm{x}_j) = \bm{b}^T K \bm{b}
\end{align*}
That is,
\begin{align*}
    L_D(f) = \bm{b}^T K \bm{b} + \|\Tilde{f}\|_{\bH}^2
\end{align*}
and
\begin{align*}
    \inf_{f\in\bH_s} L_D(f) = \|\Tilde{f}\|_{\bH}^2
\end{align*}
It follows that 
\begin{align*}
    \Delta(f) = L_D(f) -  \inf_{f\in\bH_s} L_D(f) = \bm{b}^T K \bm{b}
\end{align*}
We claim that
\begin{align*}
    \Delta_a^* = \min_{\bm{b}:\frac{1}{2n}\|K\bm{b} \|_2^2 = a } \bm{b}^T K \bm{b} = \frac{1}{\gamma_1} \|K\bm{b} \|_2^2 = 2na / \gamma_1
\end{align*}
 where the equality is obtained when $\bm{b}$ is in the direction of the largest eigenvector of $K$. To see this, we check $\|K\bm{b} \|_2^2 \leq \gamma_1\bm{b}^T K \bm{b}$. Recall the eigendecomposition of $K = G\Gamma G^{T}$ where $G = [\bm{g}_1,\ldots,\bm{g}_n]$ has orthogonal columns and $\Gamma = diag(\gamma_1,\ldots,\gamma_n)$. 
Then 
$$\|K\bm{b} \|_2^2 = \sum_{i=1}^n \gamma_i^2 (\bm{g_i}^T \bm{b})^2$$
and 
$$\bm{b}^T K\bm{b} = \sum_{i=1}^n \gamma_i (\bm{g_i}^T \bm{b})^2$$
So we have
$$\|K\bm{b} \|_2^2\leq \gamma_1[\sum_{i=1}^n \gamma_i (\bm{g_i}^T \bm{b})^2] = \gamma_1\bm{b}^T K\bm{b}$$
This finishes our claim on $\Delta_a^*$.\\
\\
\textbf{SGD output:}
By Theorem \ref{ch2:thm:informal1}, the SGD output has
$$(1-2\epsilon)\gamma_1 E[\|\bm{b}_{k_2}\|_2]\leq E[\|K \bm{b}_{k_2}\|_2]$$ 
Thus
\begin{align*}
   E[ \Delta^{1/2}(f^{SGD})] &=E\left[\sqrt{\sum_{i=1}^n \gamma_i (\bm{g_i}^T \bm{b}^{SGD})^2}\right]\\
    &\leq \sqrt{\gamma_1}E[(\sum_{i=1}^n (\bm{g_i}^T \bm{b}^{SGD})^2)^{1/2}]\\
    & = \sqrt{\gamma_1} E\|\bm{b}^{SGD}\|_2\\
    &\leq \sqrt{\gamma_1} E[\|K\bm{b}^{SGD} \|_2]/[(1 - 2\epsilon)\gamma_1]\\
    & = \sqrt{2na/\gamma_1} / (1 - 2\epsilon)\\
    &= \frac{1}{1 - 2\epsilon} (\Delta^*_a)^{1/2}\\
    & < (1 + 4\epsilon)(\Delta^*_a)^{1/2}
\end{align*}
last inequality by let $\epsilon < 1/4$.
\\
\\
\textbf{GD output:}
By Theorem \ref{ch2:thm:informal2}, the GD output has
$$\frac{\|K \bm{b}^{GD}\|^2_2}{\|\bm{b}^{GD}\|^2_2}\leq (1 + \epsilon')\gamma_n^2$$
Thus
\begin{align*}
    \Delta(f^{GD}) &=\sum_{i=1}^n \gamma_i (\bm{g_i}^T \bm{b}^{GD})^2\\
    &\geq \gamma_n\sum_{i=1}^n (\bm{g_i}^T \bm{b}^{GD})^2\\
    & = \gamma_n \|\bm{b}^{GD}\|_2^2\\
    &\geq \gamma_n\|K\bm{b}^{GD} \|_2^2/[(1 + \epsilon')\gamma_n^2]\\
    & = 2na / [(1 + \epsilon')\gamma_n]\\
    &= \frac{\gamma_1}{(1 + \epsilon')\gamma_n} \Delta^*_a\\
    & > \frac{\gamma_1}{\gamma_n} (1-\epsilon') \Delta^*_a\\
    &:= M \Delta_a^*
\end{align*}
where $M > 1$ by taking $\epsilon' < 1 - \gamma_n/\gamma_1$.

\section{Experiments}\label{ch2:app:H}
We list the implementation details of the experiments in Section \ref{ch2:sec:4} and include more experiment results. For better presenting, we split into two subsections: Subsection \ref{ch2:app:H1} includes the details of simulation; Subsection \ref{ch2:app:H2} is about the NN experiment on FashionMNIST, including the data description, network structure, and algorithm details, also there are more experiment results in Subsection \ref{ch2:app:H2} that are not listed in Section \ref{ch2:sec:4} due to page limit. 

\subsection{Simulation}\label{ch2:app:H1}
This subsection is corresponding to Figure \ref{ch2:fig:sim}. 

\textbf{Data Generation}. The training data is simulated as follows: Set $n = 10, p=100$, simulate $X_{n\times p}$ where elements of $X$ are i.i.d. $N(0,1)$; denote $i$th row of $X$ as $\bm{x}_i$, normalize $\bm{x}_i$ such that it has squared $\ell_2$ norm in $[.49,1]$; set $y_i = \sum_{j = 1}^p\sin(x_{i,j})+ \epsilon_i$ where $\epsilon_i \stackrel{i.i.d.}{\sim} N(0,.01)$. The testing data is simulated in exactly the same way, except that we only simulate $n=5$ testing data.

\textbf{Kernel Function}. We set the kernel function to be the polynomial kernel 
$$K(\bm{x}_1,\bm{x}_2) = (\langle \bm{x}_1,\bm{x}_2\rangle + .01)^2$$
\textbf{SGD and GD implementation}. 
Both SGD and GD is run for small and moderate step sizes. The moderate step size scheme for SGD is: $\eta_1 = .1$ for the first $50$ steps, and $\eta_2 = .01$ for the next $1000$ steps; for GD is: $\eta_1 = .5$ for the first $50$ steps, and $\eta_2 = .05$ for the next $1000$ steps. The small step size scheme for SGD is $\eta = 0.01$ for $1050$ steps; for GD is $\eta = 0.05$ for $1050$ steps. Note that the step size for SGD is a fraction of that for GD, this matches our Theorem \ref{ch2:thm:informal1} and \ref{ch2:thm:informal2} that the step size of GD is of magnitude $n/2$ times that of SGD.

\subsection{Neural Network on FashionMNIST}\label{ch2:app:H2}
This subsection is corresponding to Figure \ref{ch2:fig:MNIST}.

\textbf{Dataset}. The original FashionMNIST consist of $60,000$ training data and $10,000$ testing data. We randomly sample $1,500$ data from original training data for training, and use all $10,000$ original testing data for testing. All data entries are normalized to $[0,1]$. 

\textbf{Network structure}. we use a $6$-layer ResNet-like \cite{He_2016_CVPR} Neural Network, and the structure is as follows
\begin{align*}
    \text { Input } &\Rightarrow 7\times 7\text { Conv } \Rightarrow \text{BatchNorm}  \Rightarrow \text{ReLU} \Rightarrow 3\times 3\text{ MaxPool }\\
    &\Rightarrow \text{ResBlock1} \Rightarrow \text{ResBlock2} \Rightarrow \text{Global AvePool} \Rightarrow \text{FC} \Rightarrow \text{output}
\end{align*}
The Residual Blocks are as Figure $7.6.3$ in \cite{zhang2021dive} (without $1\times 1$ convolution). Note that each residual block contains two $3\times 3$ convolutional layers, thus total number of layers is as stated. 

\textbf{Algorithm}. We minimize the Cross Entropy Loss objective $L(\mathbf{w}) = \frac{1}{n}\sum_{i=1}^n l_i(\bf{w})$, where $l_i(\bf{w})$ is the loss function at $i$th sample. One step SGD is as follows:
\begin{align*}
    \bm{w}_{t + 1} = \bm{w}_t - \eta_t\frac{1}{|I|}\sum_{i\in I} \nabla l_i(\bm{w}_t)
\end{align*}
where $I$ is a randomly sampled subset of $\{1,\ldots,n\}$ (uniform random sample without replacement). We choose the batch size $|I|$ to be $25$.

One step GD is as follows:
\begin{align*}
    \bm{w}_{t + 1} = \bm{w}_t - \eta_t \nabla L(\bm{w}_t)
\end{align*}
Both SGD and GD are run using two settings of step sizes $\eta_t$. The moderate step size setting is as follows:
\begin{align*}
    \eta_t = \begin{cases} 0.2, & t = 1,\ldots,5000\\ 0.02, & t = 5001,\ldots,20000\end{cases}
\end{align*}
And the small step size setting has $\eta_t = 0.02, t = 1,\ldots,20000$.

\textbf{Comparison of convergence direction}. Since the loss surface is nonconvex and the Hessian varies, we follow \cite{wu2021direction} to measure the convergence direction by Relative Rayleigh Quotient(RRQ), which normalizes the Rayleigh Quotient by the maximum eigenvalue of the Hessian as follows
$$RRQ(\bm{w}) = \frac{\frac{\nabla L(\bm{w})^{\top}}{\|\nabla L(\bm{w})\|_{2}} \cdot \nabla^{2} L(\bm{w}) \cdot \frac{\nabla L(\bm{w})}{\|\nabla L(\bm{w})\|_{2}}}{\left\|\nabla^{2} L(\bm{w})\right\|_{2}}$$
where $L(w)$ is the loss function on the whole training set. A high RRQ indicates that the convergence direction of $w$ is close to a larger eigenvector of the Hessian. 

\textbf{Comparison of test accuracy}. We set $20$ different random seeds. For each random seed, we run: SGD with moderate step size, GD with moderate step size, SGD with small step size, GD with small step size. For each algorithm, we evaluate its test accuracy once every $500$ steps, and use the average of the last $5$ values as its test accuracy. We list the test accuracy in Table \ref{ch2:tab:acc}.
\begin{table}[H]\small
    \centering
    \begin{tabular}{c||cccccccccc}\hline
    Experiment & \#1 & \#2 & \#3 & \#4 & \#5 & \#6 & \#7 & \#8 & \#9 & \#10\\\hline
        SGD + moderate LR & 83.69& 82.95& 82.37& 82.05& 83.4 & 83.16& 83.72& 83.29& 83.28& 83.23 \\\hline
        GD + moderate LR & 80.93& 80.79& 80.79& 81.80 & 81.68& 81.12& 82.43& 81.63& 80.94& 81.54 \\\hline
        SGD + small LR &82.00  & 81.72& 81.34& 81.92& 82.63& 82.67& 82.99& 82.22& 80.78& 82.10 \\\hline
        GD + small LR & 78.88& 78.71& 78.49& 79.3 & 80.45& 79.78& 80.15& 79.66& 79.54& 79.68\\\hline\hline
        Experiment & \#11 & \#12 & \#13 & \#14 & \#15 & \#16 & \#17 & \#18 & \#19 & \#20\\\hline
        SGD + moderate LR & 83.12& 82.92& 83.58& 83.47& 82.35& 83.57& 83.59& 82.43& 84.21& 83.12 \\\hline
        GD + moderate LR & 82.41& 81.56& 81.42& 80.86& 81.23& 81.25& 81.82& 80.42& 81.80 & 82.12\\\hline
        SGD + small LR &82.62& 80.66& 82.01& 81.01& 81.32& 81.66& 82.12& 80.78& 82.28& 82.48\\\hline
        GD + small LR &80.08& 78.29& 79.93& 79.36& 78.9 & 79.69& 80.2 & 79.62& 79.98& 79.69\\\hline
    \end{tabular}
    \caption{Test Accuracy}
    \label{ch2:tab:acc}
\end{table}
We also use one-side Wilcoxon signed-rank test to check if the test accuracy of different algorithm are significantly different, the result is in Table \ref{ch2:tab:t_test}. All the p-values are significant at $0.01$ level, so we reject the null hypothesis and conclude that the SGD with moderate step size has test accuracy significantly higher than all other algorithms.
\begin{table}[H]
    \centering
    \begin{tabular}{|c|c|}\hline
       Null Hypothesis on Test Accuracy  & p-value \\\hline
       SGD + moderate LR $\leq$ GD + moderate LR & $9.54 \times 10^{-7}$ \\\hline
       SGD + moderate LR $\leq$ SGD + small LR & $9.54 \times 10^{-7}$\\\hline
       SGD + moderate LR $\leq$ GD + small LR & $9.54 \times 10^{-7}$\\\hline
    \end{tabular}
    \caption{Wilcoxon signed-rank test result}
    \label{ch2:tab:t_test}
\end{table}
\textbf{Additional experiments}.  We conduct more experiments using different step sizes. The initial step size is taken in $\{1,0.5,0.2,0.1,0.02,0.01,0.005,0.001\}$, and the step size is divided by a factor of $10$ after $5000$ steps. The test accuracy is in Figure \ref{ch2:fig:acc_more}, where we see that SGD with step size $0.2$ has the best test accuracy, and GD with step size $0.5$ performs better than GD with any other step sizes, but is still worse than the best SGD. 
\begin{figure}
    \centering
    \includegraphics[scale = 0.25]{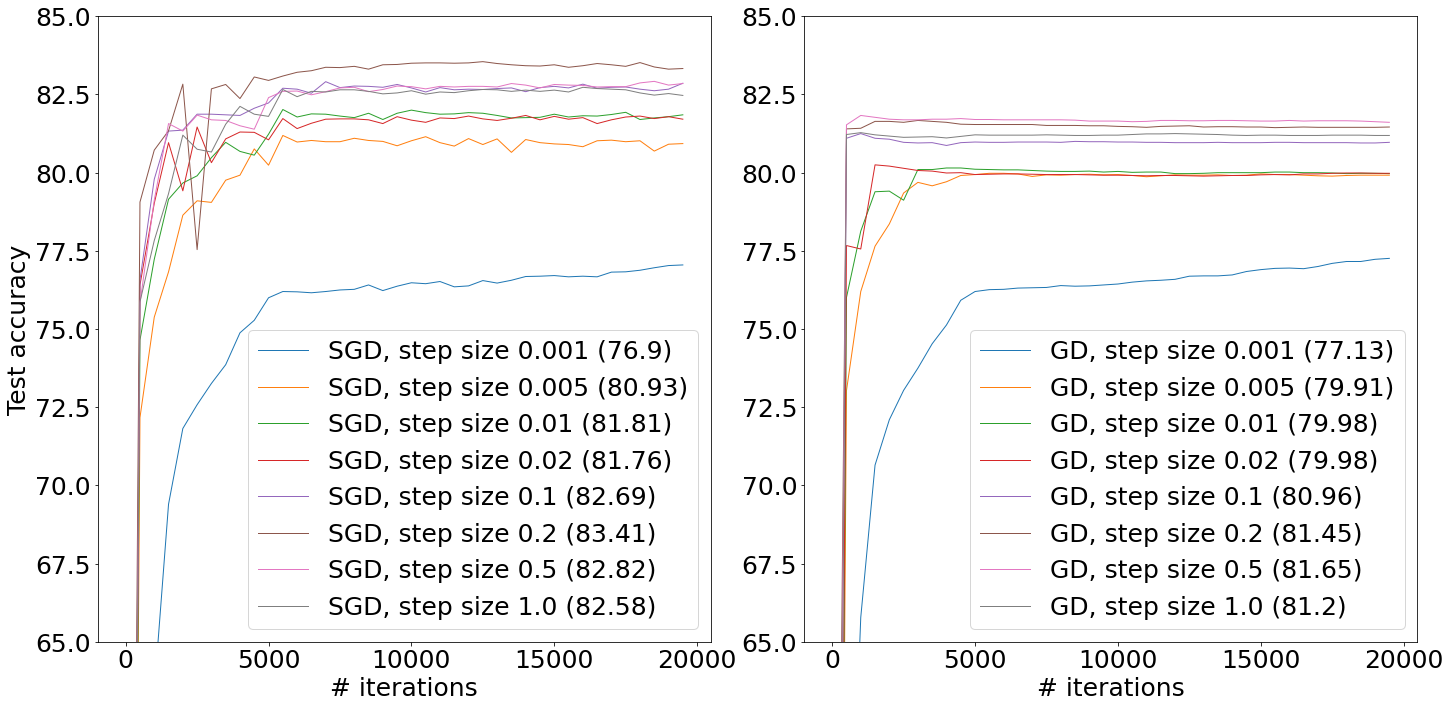}
    \caption{Use more step sizes in SGD/GD. The test accuracy is evaluated once every $500$ iterations, and inside the bracket is the average of the last $5$ test accuracy values.}
    \label{ch2:fig:acc_more}
\end{figure}

\end{document}